\documentclass[11pt]{article}

\usepackage{booktabs}       

\usepackage[margin=1in]{geometry}
\usepackage{times}
\usepackage{amsmath,amsfonts,amsthm,amssymb,xspace,bm, verbatim,dsfont,mathtools}
\usepackage{graphicx}
\usepackage{url, algorithm, algorithmic} 
\usepackage{subfigure}
\usepackage{multirow}

\usepackage{color}
\usepackage{epstopdf}
\usepackage{appendix}

\theoremstyle{plain}
\newtheorem{theorem}{Theorem}[section]
\newtheorem{lemma}[theorem]{Lemma}
\newtheorem*{lemma*}{Lemma}

\newtheorem{definition}[theorem]{Definition}
\theoremstyle{definition}

\newcommand{\eps}{\epsilon}

\DeclareMathOperator*{\argmin}{argmin}

\newcommand{\Vht}{\widehat{V}^{(t+1)}}

\newcommand{\algo}{\textsc{SMP-PCA}\xspace}
\newcommand{\prevalgo}{\textsc{LELA}\xspace}

\title{Single Pass PCA of Matrix Products}

\author{
Shanshan Wu\\
{The University of Texas at Austin}\\
\texttt{shanshan@utexas.edu}
\and
Srinadh Bhojanapalli\\
{Toyota Technological Institute at Chicago}\\
\texttt{srinadh@ttic.edu}\\
\and
Sujay Sanghavi\\
{The University of Texas at Austin}\\
\texttt{sanghavi@mail.utexas.edu}
\and
Alexandros G.~Dimakis\\
{The University of Texas at Austin}\\
\texttt{dimakis@austin.utexas.edu}
}

\begin{document}
\maketitle
\begin{abstract}
In this paper we present a new algorithm for computing a low rank approximation of the product $A^TB$ by taking only a single pass of the
two matrices $A$ and $B$. The straightforward way to do this is to (a) first sketch $A$ and $B$ individually, and then (b) find the top components using PCA on the sketch. Our algorithm in contrast retains additional summary information about $A,B$ (e.g. row and column norms etc.) and uses this additional information to obtain an improved approximation from the sketches. Our main analytical result establishes a comparable spectral norm guarantee to existing two-pass methods; in addition we also provide results from an Apache Spark implementation that shows better computational and statistical performance on real-world and synthetic evaluation datasets.
\end{abstract}
\section{Introduction}\label{intro}
Given two large matrices $A$ and $B$ we study the problem of finding a low rank approximation of their product $A^T B$, using only one pass over the matrix elements.
This problem has many applications in machine learning and statistics. For example, if $A=B$, then this general problem reduces to Principal Component Analysis (PCA). Another example is a low rank approximation of a co-occurrence matrix from large logs, e.g., $A$ may be a user-by-query matrix and $B$ may be a user-by-ad matrix, so $A^TB$ computes the joint counts for each query-ad pair. The matrices $A$ and $B$ can also be two large bag-of-word matrices. For this case, each entry of $A^TB$ is the number of times a pair of words co-occurred together. As a fourth example, $A^TB$ can be a cross-covariance matrix between two sets of variables, e.g., $A$ and $B$ may be genotype and phenotype data collected on the same set of observations. A low rank approximation of the product matrix is useful for Canonical Correlation Analysis (CCA)~\cite{CCA}. For all these examples, $A^TB$ captures pairwise variable interactions and a low rank approximation is a way to efficiently represent the significant pairwise interactions in sub-quadratic space.  

Let $A$ and $B$ be matrices of size $d\times n$ ($d\gg n$) assumed too large to fit in main memory. To obtain a rank-$r$ approximation of $A^T B$, a naive way is to compute $A^T B$ first, and then perform truncated singular value decomposition (SVD) of $A^T B$. This algorithm needs $O(n^2d)$ time and $O(n^2)$ memory to compute the product, followed by an SVD of the $n\times n$ matrix. An alternative option is to directly run power method on $A^T B$ without explicitly computing the product. Such an algorithm will need to access the data matrices $A$ and $B$ multiple times and the disk IO overhead for loading the matrices to memory multiple times will be the major performance bottleneck. 

For this reason, a number of recent papers introduce randomized algorithms that require only a few passes over the data, approximately linear memory, and also provide spectral norm guarantees. The key step in these algorithms is to compute a smaller representation of data. This can be achieved by two different methods: (1) dimensionality reduction, i.e., matrix sketching~\cite{sarlos, clarkson2013low, magen2011low, stablerank}; (2) random sampling~\cite{drineas2006fast, leverage}. The recent results of Cohen et al.~\cite{stablerank} provide the strongest spectral norm guarantee of the former. They show that a sketch size of $O(\tilde{r}/\eps^2)$ suffices for the sketched matrices $\widetilde{A}^T\widetilde{B}$ to achieve a spectral error of $\epsilon$, where $\tilde{r}$ is the maximum stable rank of $A$ and $B$. Note that $\widetilde{A}^T\widetilde{B}$ is not the desired rank-$r$ approximation of $A^TB$. On the other hand, ~\cite{leverage} is a recent sampling method with very good performance guarantees. The authors consider entrywise sampling based on column norms, followed by a matrix completion step to compute low rank approximations. There is also a lot of interesting work on streaming PCA, but none can be directly applied to the general case when $A$ is different from $B$ (see Figure~\ref{svdAsvdB}). Please refer to Appendix~\ref{sec:related} for more discussions on related work.

Despite the significant volume of prior work, there is no method that computes a rank-$r$ approximation of $A^TB$ when the entries of $A$ and $B$ are streaming in a single pass
\footnote{One straightforward idea is to sketch each matrix individually and perform SVD on the product of the sketches. We compare against that scheme and show that we can perform arbitrarily better using our rescaled JL embedding.}. Bhojanapalli et al.~\cite{leverage} consider a two-pass algorithm which computes column norms in the first pass and uses them to sample in a second pass over the matrix elements. In this paper, we combine ideas from the sketching and sampling literature 
to obtain the first algorithm that requires only a single pass over the data. 

{\bf Contributions:} 
\begin{itemize}
\item We propose a one-pass algorithm~\algo (which stands for Streaming Matrix Product PCA) that computes a rank-$r$ approximation of $A^T B$ in time $O((\text{nnz}(A)+\text{nnz}(B))\frac{\rho^2 r^3\tilde{r}}{\eta^2}+ \frac{nr^6\rho^4\tilde{r}^3}{\eta^4})$. Here $\text{nnz}(\cdot)$ is the number of non-zero entries, $\rho$ is the condition number, $\tilde{r}$ is the maximum stable rank, and $\eta$ measures the spectral norm error. Existing two-pass algorithms such as~\cite{leverage} typically have longer runtime than our algorithm (see Figure~\ref{runtime}). We also compare our algorithm with the simple idea that first sketches $A$ and $B$ separately and then performs SVD on the product of their sketches. We show that our algorithm \textit{always} achieves better accuracy and can perform arbitrarily better if the column vectors of $A$ and $B$ come from a cone (see Figures~\ref{tildeM}, ~\ref{svdOnepass}, ~\ref{siftbag}). 
 
\item The central idea of our algorithm is a novel \textit{rescaled JL embedding} that combines information from matrix sketches and vector norms. This allows us to get better estimates of dot products of high dimensional vectors compared to previous sketching approaches. We explain the benefit compared to a naive JL embedding in Figure~\ref{tildeM} and the related discussion; we believe it may be of more general interest beyond low rank matrix approximations. 

\item We prove that our algorithm recovers a low rank approximation of $A^T B$ up to an error that depends on $\|A^T B -(A^TB)_r \|$ and $\|A^TB\|$, decaying with increasing sketch size and number of samples (Theorem~\ref{mainTheorem}). The first term is a consequence of low rank approximation and vanishes if $A^T B$ is exactly rank-$r$. The second term results from matrix sketching and subsampling; the bounds have similar dependencies as in~\cite{stablerank}. 

\item We implement \algo in Apache Spark and perform several distributed experiments on synthetic and real datasets. 
Our distributed implementation uses several design innovations described in Section~\ref{experiments} and Appendix~\ref{sec:app_samp} and it is the only Spark implementation that we are aware of that can handle matrices that are large in both dimensions. 
Our experiments show that we improve by approximately a factor of $2\times$ in running time compared to the previous state of the art and scale gracefully as the cluster size increases. The source code is available online~\cite{github}.

\item In addition to better performance, our algorithm offers another advantage: It is possible to compute low-rank approximations to $A^T B$ even when the entries of the two matrices arrive in some arbitrary order (as would be the case in streaming logs).
We can therefore discover significant correlations even when the original datasets cannot be stored, for example due to storage or privacy limitations. 
\end{itemize}
\section{Problem setting and algorithms}\label{algorithms}
Consider the following problem: given two matrices $A\in \mathbb{R}^{d\times n_1}$ and $B\in \mathbb{R}^{d\times n_2}$ that are stored in disk, find a rank-$r$ approximation of their product $A^TB$. In particular, we are interested in the setting where both $A$, $B$ and $A^TB$ are too large to fit into memory. This is common for modern large scale machine learning applications. For this setting, we develop a single-pass algorithm \algo that computes the rank-$r$ approximation without explicitly forming the entire matrix $A^TB$. 

{\bf Notations.}
Throughout the paper, we use $A(i,j)$ or $A_{ij}$ to denote $(i,j)$ entry for any matrix $A$. Let $A_i$ and $A^j$ be the $i$-th column vector and $j$-th row vector. We use $\|A\|_F$ for Frobenius norm, and $\|A\|$ for spectral (or operator) norm. The optimal rank-$r$ approximation of matrix $A$ is $A_r$, which can be found by SVD. Given a set $\Omega \subset [n_1]\times [n_2]$ and a matrix $A\in \mathbb{R}^{n_1\times n_2}$, we define $P_\Omega(A)\in \mathbb{R}^{n_1\times n_2}$ as the projection of $A$ on $\Omega$, i.e., $P_\Omega(A)(i,j) = A(i,j)$ if $(i,j)\in\Omega$ and 0 otherwise.

\subsection{\algo}
Our algorithm \algo (Streaming Matrix Product PCA) takes four parameters as input: the desired rank $r$, number of samples $m$, sketch size $k$, and the number of iterations $T$. Performance guarantee involving these parameters is provided in Theorem~\ref{mainTheorem}. As illustrated in Figure~\ref{overview}, our algorithm has three main steps: 1) compute sketches and side information in one pass over $A$ and $B$; 2) given partial information of $A$ and $B$, estimate {\it important} entries of $A^T B$; 3) compute low rank approximation given estimates of a few entries of $A^T B$. Now we explain each step in detail.

\begin{figure*}[ht]
\centering
\includegraphics[width=\textwidth]{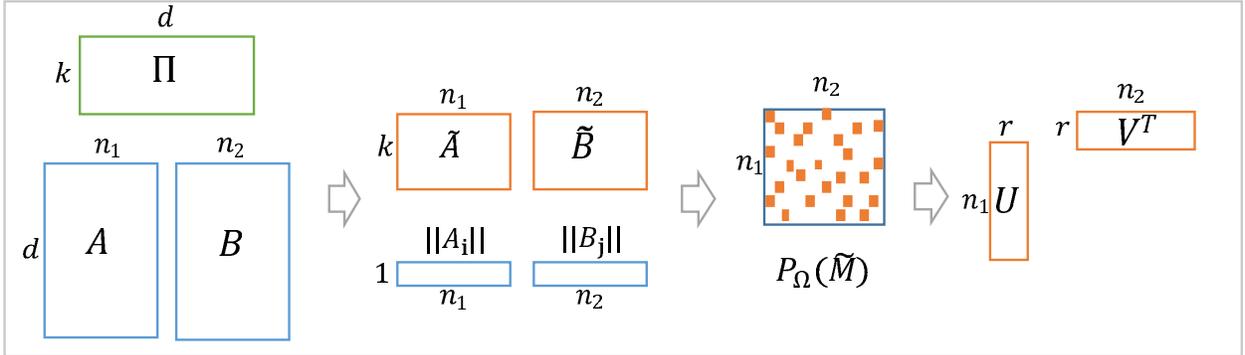}
\caption{An overview of our algorithm. A single pass is performed over the data to produce the sketched matrices $\widetilde{A}$, $\widetilde{B}$ and the column norms $\|A_i\|$, $\|B_j\|$, for all $(i,j)\in [n_1]\times[n_2]$. We then compute the sampled matrix $P_{\Omega}(\widetilde{M})$ through a biased sampling process, where $P_{\Omega}(\widetilde{M}) = \widetilde{M}(i,j)$ if $(i,j)\in \Omega$ and zero otherwise. Here $\Omega$ represents the set of sampled entries. We define $\widetilde{M}$ as an estimator for $A^TB$, and compute its entry as $\widetilde{M}(i,j) = \|A_i\|\cdot\|B_j\|\cdot\frac{\widetilde{A}_i^T\widetilde{B}_j}{\|\widetilde{A}_i\|\cdot \|\widetilde{B}_j\|}$. Performing matrix completion on $P_{\Omega}(\widetilde{M})$ gives the desired rank-$r$ approximation.}\label{overview}
\end{figure*}

\begin{algorithm}[ht]
   \caption{\algo: Streaming Matrix Product PCA}
   \label{OnePassLELA}
\begin{algorithmic}[1]
   \STATE {\bfseries Input:} $A\in \mathbb{R}^{d\times n_1}$, $B\in \mathbb{R}^{d\times n_2}$, desired rank: $r$, sketch size: $k$, number of samples: $m$, number of iterations: $T$ 
   \STATE Construct a random matrix $\Pi\in \mathbb{R}^{k\times d}$, where $\Pi(i,j)\sim \mathcal{N}(0, 1/k)$, $\forall (i,j)\in [k]\times[d]$. Perform a single pass over $A$ and $B$ to obtain: $\widetilde{A}=\Pi A$, $\widetilde{B}=\Pi B$, and $\|A_i\|$, $\|B_j\|$, $\forall (i,j)\in [n_1]\times [n_2]$.
   \STATE Sample each entry $(i,j)\in [n_1]\times [n_2]$ independently with probability $\hat{q}_{ij}=\min\{1, q_{ij}\}$, where $q_{ij}$ is defined in Eq.(\ref{q_ij}); maintain a set $\Omega \subset [n_1]\times [n_2]$ which stores all the sampled pairs $(i,j)$.
   \STATE Define $\widetilde{M}\in\mathbb{R}^{n_1 \times n_2}$, where $\widetilde{M}(i,j)$ is given in Eq. (\ref{dot-approx}). Calculate $P_\Omega(\widetilde{M})\in\mathbb{R}^{n_1 \times n_2}$, where $P_{\Omega}(\widetilde{M}) = \widetilde{M}(i,j)$ if $(i,j)\in \Omega$ and zero otherwise.
   \STATE Run WAltMin($P_\Omega(\widetilde{M})$, $\Omega$, $r$, $\hat{q}$, $T$ ), see Appendix~\ref{appWAltMin} for more details.\\
   \STATE {\bfseries Output:} $\widehat{U}\in \mathbb{R}^{n_1\times r}$ and $\widehat{V}\in\mathbb{R}^{n_2\times r}$. 
\end{algorithmic}
\end{algorithm}

{\bf Step 1: Compute sketches and side information in one pass over $A$ and $B$.}
In this step we compute sketches $\widetilde{A}:=\Pi A$ and $\widetilde{B}:=\Pi B$, where $\Pi\in \mathbb{R}^{k\times d}$ is a random matrix with entries being i.i.d. $\mathcal{N}(0, 1/k)$ random variables. It is known that $\Pi$ satisfies an "oblivious Johnson-Lindenstrauss (JL) guarantee"~\cite{sarlos}\cite{sketching} and it helps preserving the top row spaces of $A$ and $B$~\cite{clarkson2013low}. Note that any sketching matrix $\Pi$ that is an oblivious subspace embedding can be considered here, e.g., sparse JL transform and randomized Hadamard transform (see~\cite{stablerank} for more discussion).

Besides $\widetilde{A}$ and $\widetilde{B}$, we also compute the $L_2$ norms for all column vectors, i.e., $\|A_i \|$ and $\| B_j \|$, for all $(i, j) \in [n_1] \times [n_2]$. We use this additional information to design better estimates of $A^T B$ in the next step, and also to determine {\it important} entries of $\widetilde{A}^T \widetilde{B}$ to sample. Note that this is the only step that needs one pass over data.

{\bf Step 2: Estimate important entries of $A^T B$ by rescaled JL embedding.} 
In this step we use partial information obtained from the previous step to compute a few important entries of $\widetilde{A}^T \widetilde{B}$. We first determine what entries of $\widetilde{A}^T \widetilde{B}$ to sample, and then propose a novel rescaled JL embedding for estimating those entries.

We sample entry $(i,j)$ of $A^TB$ independently with probability $\hat{q}_{ij} = \min\{1, q_{ij}\}$, where 
\begin{equation}
q_{ij} = m \cdot (\frac{\|A_i\|^2}{2n_2\|A\|^2_F}+ \frac{\|B_j\|^2}{2n_1\|B\|^2_F}). \label{q_ij}
\end{equation}
Let $\Omega \subset [n_1]\times [n_2]$ be the set of sampled entries $(i,j)$. Since $\mathbb{E}(\sum_{i,j}q_{ij})=m$, the expected number of sampled entries is roughly $m$. The special form of $q_{ij}$ ensures that we can draw $m$ samples in $O(n_1+m\log(n_2))$ time; we show how to do this in Appendix~\ref{sec:app_samp}. 

Note that $q_{ij}$ intuitively captures important entries of $A^T B$ by giving higher weight to heavy rows and columns. We show in Section~\ref{sec:analysis} that this sampling actually generates good approximation to the matrix $A^T B$. 

The biased sampling distribution of Eq. (\ref{q_ij}) is first proposed by Bhojanapalli et al.~\cite{leverage}. However, their algorithm~\cite{leverage} needs a second pass to compute the sampled entries, while we propose a novel way of estimating dot products, using information obtained in the first step.  

Define $\widetilde{M}\in\mathbb{R}^{n_1 \times n_2}$ as
\begin{equation}
\widetilde{M}(i,j) = \|A_i\|\cdot\|B_j\|\cdot\frac{\widetilde{A}_i^T\widetilde{B}_j}{\|\widetilde{A}_i\|\cdot \|\widetilde{B}_j\|} \label{dot-approx}.
\end{equation} 
Note that we will not compute and store $\widetilde{M}$, instead, we only calculate $\widetilde{M}(i,j)$ for $(i,j)\in\Omega$. This matrix is denoted as $P_\Omega(\widetilde{M})$, where $P_\Omega(\widetilde{M})(i,j) = \widetilde{M}(i,j)$ if $(i,j)\in\Omega$ and 0 otherwise.

\begin{figure*}[ht]
\centering
\subfigure[]{
	\includegraphics[width=0.5\textwidth]{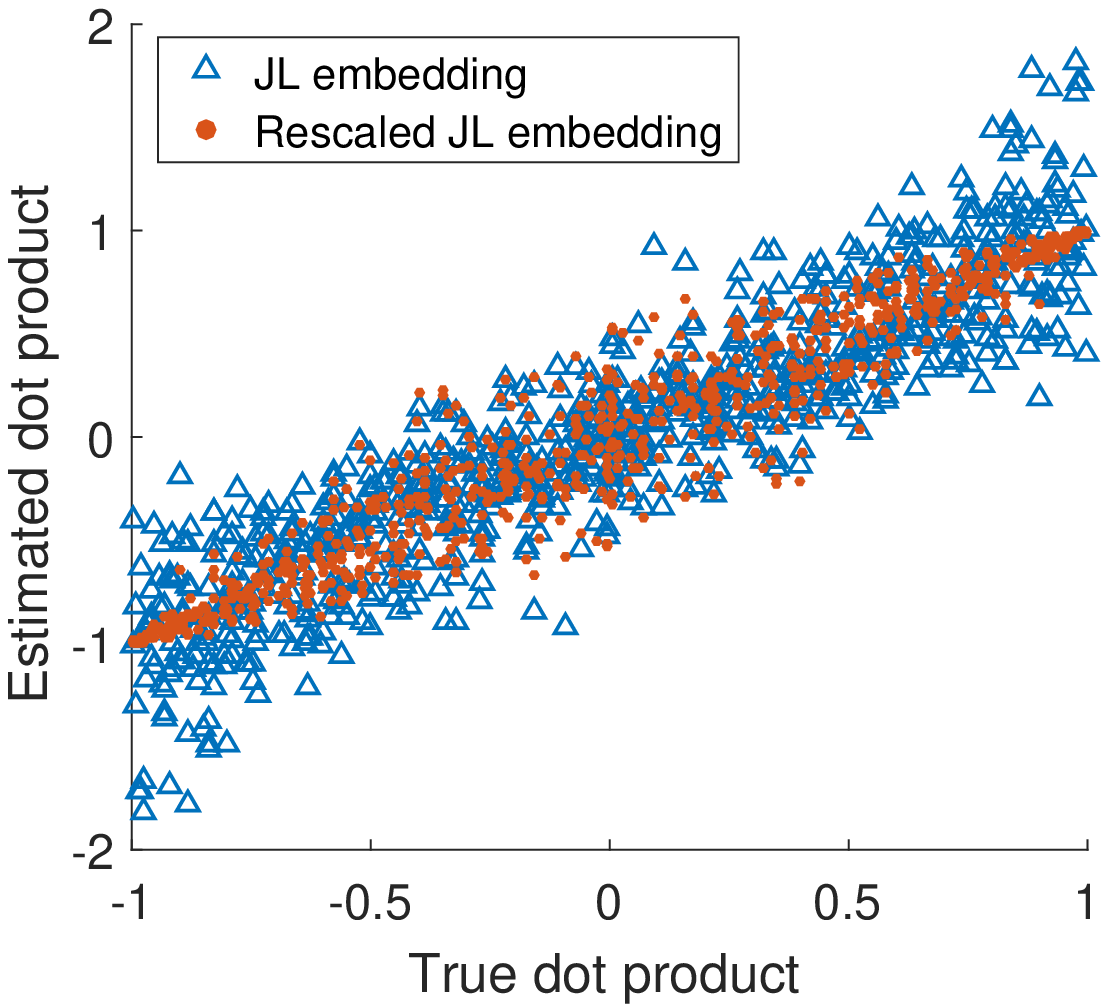}
	\label{rescale}}
\subfigure[]{
	\includegraphics[width=0.46\textwidth]{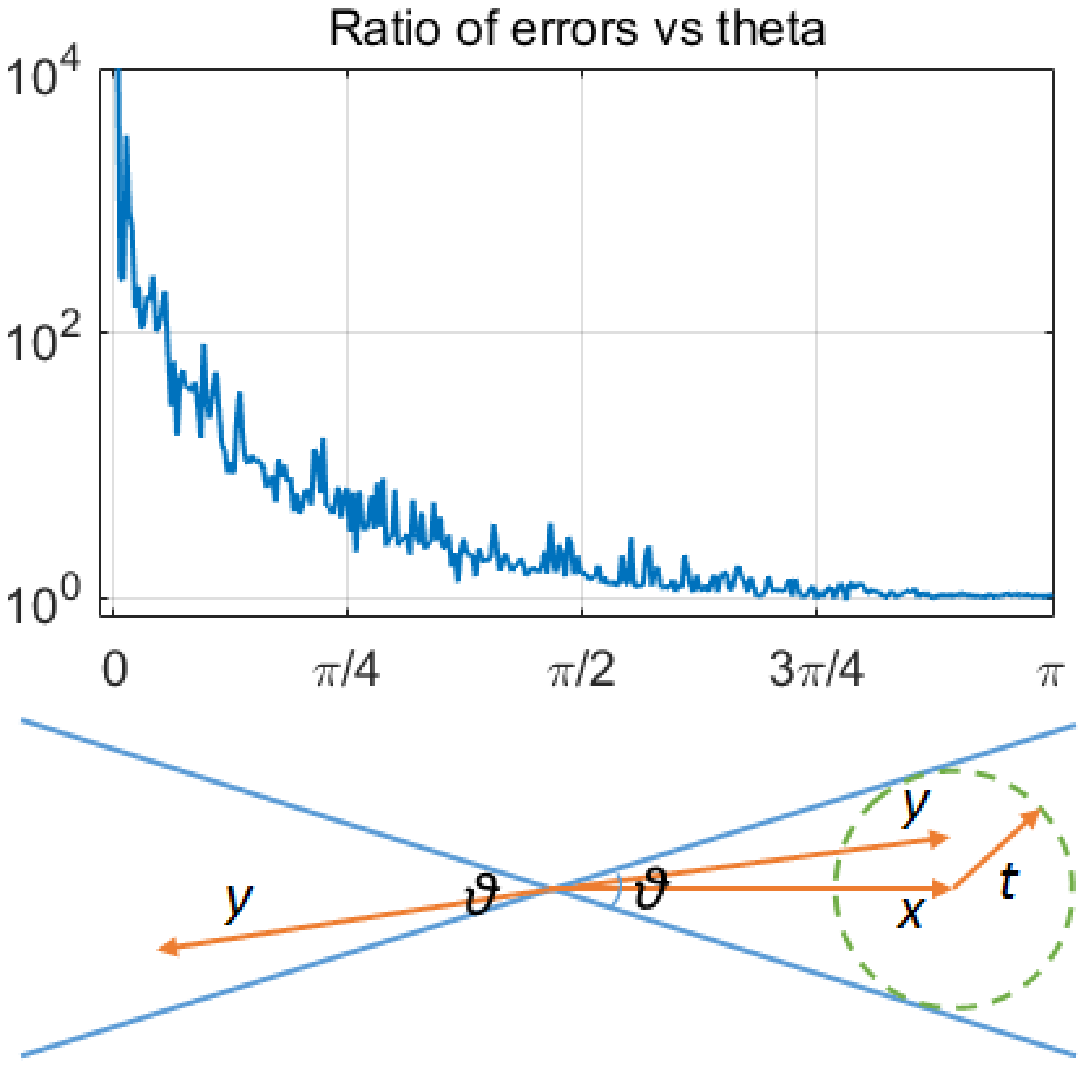}
	\label{cone}}
\caption{(a) Rescaled JL embedding (red dots) captures the dot products with smaller variance compared to JL embedding (blue triangles). Mean squared error: 0.053 versus 0.129. (b) Lower figure illustrates how to construct unit-norm vectors from a cone with angle $\theta$. Let $x$ be a fixed unit-norm vector, and let $t$ be a random Gaussian vector with expected norm $\tan(\theta/2)$, we set $y$ as either $x+t$ or $-(x+t)$ with probability half, and then normalize it. Upper figure plots the ratio of spectral norm errors $\|A^TB-\widetilde{A}^T\widetilde{B}\|/\|A^TB-\widetilde{M}\|$, when the column vectors of $A$ and $B$ are unit vectors drawn from a cone with angle $\theta$. Clearly, $\widetilde{M}$ has better accuracy than $\widetilde{A}^T\widetilde{B}$ for all possible values of $\theta$, especially when $\theta$ is small.}\label{tildeM}
\end{figure*}

We now explain the intuition of Eq. (\ref{dot-approx}), and why $\widetilde{M}$ is a better estimator than $\widetilde{A}^T\widetilde{B}$. To estimate the $(i,j)$ entry of $A^TB$, a straightforward way is to use $\widetilde{A}_i^T \widetilde{B}_j=\|\widetilde{A}_i\|\cdot \|\widetilde{B}_j\| \cdot \cos{\widetilde{\theta}_{ij}}$, where $\widetilde{\theta}_{ij}$ is the angle between vectors $\widetilde{A}_i$ and $\widetilde{B}_j$. Since we already know the actual column norms, a potentially better estimator would be $\|A_i\|\cdot \|B_j\| \cdot \cos{\widetilde{\theta}_{ij}}$. This removes the uncertainty that comes from distorted column norms\footnote{We also tried using the cosine rule for computing the dot product, and another sketching method specifically designed for preserving angles~\cite{angle}, but empirically those methods perform worse than our current estimator.}. 

Figure~\ref{rescale} compares the two estimators $\widetilde{A}_i^T \widetilde{B}_j$ (JL embedding) and $\widetilde{M}(i,j)$ (rescaled JL embedding) for dot products. We plot simulation results on pairs of unit-norm vectors with different angles. The vectors have dimension 1,000 and the sketching matrix has dimension 10-by-1,000. Clearly rescaling by the actual norms help reduce the estimation uncertainty. This phenomenon is more prominent when the true dot products are close to $\pm1$, which makes sense because $\cos{\theta}$ has a small slope when $\cos{\theta}$ approaches $\pm1$, and hence the uncertainty from angles may produce smaller distortion compared to that from norms. In the extreme case when $\cos{\theta}=\pm 1$, rescaled JL embedding can perfectly recover the true dot product. 

In the lower part of Figure~\ref{cone} we illustrate how to construct unit-norm vectors from a cone with angle $\theta$. Given a fixed unit-norm vector $x$, and a random Gaussian vector $t$ with expected norm $\tan(\theta/2)$, we construct new vector $y$ by randomly picking one from the two possible choices $x+t$ and $-(x+t)$, and then renormalize it. Suppose the columns of $A$ and $B$ are unit vectors randomly drawn from a cone with angle $\theta$, we plot the ratio of spectral norm errors $\|A^TB-\widetilde{A}^T\widetilde{B}\|/\|A^TB-\widetilde{M}\|$ in Figure~\ref{cone}. We observe that $\widetilde{M}$ always outperforms $\widetilde{A}^T\widetilde{B}$ and can be much better when $\theta$ approaches zero, which agrees with the trend indicated in Figure~\ref{rescale}.

{\bf Step 3: Compute low rank approximation given estimates of few entries of $A^T B$.}  
Finally we compute the low rank approximation of $A^T B$ from the samples using alternating least squares:
\begin{equation}
\min_{U,V\in\mathbb{R}^{n\times r}} \sum_{(i,j)\in \Omega} w_{ij}(e_i^T UV^T e_j-\widetilde{M}(i,j))^2,
\end{equation}
where $w_{ij}=1/\hat{q}_{ij}$ denotes the weights, and $e_i$, $e_j$ are standard base vectors. This is a popular technique for low rank recovery and matrix completion (see~\cite{leverage} and the references therein). After $T$ iterations, we will get a rank-$r$ approximation of $\widetilde{M}$ presented in the convenient factored form. This subroutine is quite standard, so we defer the details to Appendix~\ref{appWAltMin}.

\section{Analysis}\label{sec:analysis}
Now we present the main theoretical result. Theorem~\ref{mainTheorem} characterizes the interaction between the sketch size $k$, the sampling complexity $m$, the number of iterations $T$, and the spectral error $\|(A^TB)_r-\widehat{A^TB}_r\|$, where $\widehat{A^TB}_r$ is the output of \algo, and $(A^TB)_r$ is the optimal rank-$r$ approximation of $A^TB$. Note that the following theorem assumes that $A$ and $B$ have the same size. For the general case of $n_1\ne n_2$, Theorem~\ref{mainTheorem} is still valid by setting $n = \max\{n_1,n_2\}$.

\begin{theorem} \label{mainTheorem}
Given matrices $A\in \mathbb{R}^{d\times n}$ and $B\in \mathbb{R}^{d\times n}$, let $(A^TB)_r$ be the optimal rank-$r$ approximation of $A^TB$. Define $\tilde{r}= \max\{\frac{\|A\|^2_F}{\|A\|^2}, \frac{\|B\|^2_F}{\|B\|^2}\}$ as the maximum stable rank, and $\rho=\frac{\sigma^*_1}{\sigma^*_r}$ as the condition number of $(A^TB)_r$, where $\sigma^*_i$ is the $i$-th singular values of $A^TB$. 

Let $\widehat{A^TB}_r$ be the output of Algorithm \algo. If the input parameters $k$, $m$, and $T$ satisfy
\begin{equation}
k \ge\frac{C_1\|A\|^2\|B\|^2\rho^2r^3}{\|A^TB\|_F^2} \cdot \frac{\max\{\tilde{r}, 2\log(n)\}+\log{(3/\gamma)}}{\eta^2}, \label{dim}
\end{equation}
\begin{equation}
m \ge \frac{C_2\tilde{r}^2}{\gamma} \cdot \left(\frac{\|{A}\|^2_F+\|{B}\|^2_F}{\|{A}^T{B}\|_F}\right)^2\cdot \frac{nr^3\rho^2\log(n)T^2}{\eta^2} , \label{sam}
\end{equation}
\begin{equation}
T \ge \log(\frac{\|A\|_F+\|B\|_F}{\zeta}),
\end{equation}
where $C_1$ and $C_2$ are some global constants independent of $A$ and $B$. Then with probability at least $1-\gamma$, we have
\begin{equation}
\|(A^TB)_r-\widehat{A^TB}_r\| \le \eta\|A^TB - (A^TB)_r\|_F + \zeta + \eta \sigma_r^*. \label{errBound}
\end{equation} 
\end{theorem}
{\bf Remark 1.} Compared to the two-pass algorithm proposed by~\cite{leverage}, we notice that Eq. (\ref{errBound}) contains an additional error term $\eta \sigma_r^*$. This extra term captures the cost incurred when we are approximating entries of $A^TB$ by Eq. (\ref{dot-approx}) instead of using the actual values. The exact tradeoff between $\eta$ and $k$ is given by Eq. (\ref{dim}). On one hand, we want to have a small $k$ so that the sketched matrices can fit into memory. On the other hand, the parameter $k$ controls how much information is lost during sketching, and a larger $k$ gives a more accurate estimation of the inner products. 

{\bf Remark 2.} The dependence on $\frac{\|{A}\|^2_F+\|{B}\|^2_F}{\|{A}^T{B}\|_F}$ captures one difficult situation for our algorithm. If $\|A^TB\|_F$ is much smaller than $\|A\|_F$ or $\|B\|_F$, which could happen, e.g., when many column vectors of $A$ are orthogonal to those of $B$, then \algo requires many samples to work. This is reasonable. Imagine that $A^TB$ is close to an identity matrix, then it may be hard to tell it from an all-zero matrix without enough samples. Nevertheless, removing this dependence is an interesting direction for future research.

{\bf Remark 3.} For the special case of $A=B$, \algo computes a rank-$r$ approximation of matrix $A^TA$ in a single pass. Theorem~\ref{mainTheorem} provides an error bound in spectral norm for the residual matrix $(A^TA)_r-\widehat{A^TA}_r$. Most results in the online PCA literature use Frobenius norm as performance measure. Recently, \cite{onlinePCA} provides an online PCA algorithm with spectral norm guarantee. They achieves a spectral norm bound of $\epsilon \sigma_1^*+\sigma_{r+1}^*$, which is stronger than ours. However, their algorithm requires a target dimension of $O(r\log n / \epsilon^2)$, i.e., the output is a matrix of size $n$-by-$O(r\log n / \epsilon^2)$, while the output of \algo is simply $n$-by-$r$.

{\bf Remark 4.} We defer our proofs to Appendix~\ref{mainProof}. The proof proceeds in three steps. In Appendix~\ref{proofInit}, we show that the sampled matrix provides a good approximation of the actual matrix $A^TB$. In Appendix~\ref{proofDist}, we show that there is a geometric decrease in the distance between the computed subspaces $\widehat{U}$, $\widehat{V}$ and the optimal ones $U^*$, $V^*$ at each iteration of WAltMin algorithm. The spectral norm bound in Theorem~\ref{mainTheorem} is then proved using results from the previous two steps. 

{\bf Computation Complexity.}
We now analyze the computation complexity of \algo. In Step 1, we compute the sketched matrices of $A$ and $B$, which requires $O(\text{nnz}(A)k+\text{nnz}(B)k)$ flops. Here $\text{nnz}(\cdot)$ denotes the number of non-zero entries. The main job of Step 2 is to sample a set $\Omega$ and calculate the corresponding inner products, which takes $O(m\log(n)+mk)$ flops. Here we define $n$ as $\max\{n_1,n_2\}$ for simplicity. According to Eq.~(\ref{dim}), we have $\log(n)=O(k)$, so Step 2 takes $O(mk)$ flops. In Step 3, we run alternating least squares on the sampled matrix, which can be completed in $O((mr^2+nr^3)T)$ flops. Since Eq.~(\ref{sam}) indicates $nr= O(m)$, the computation complexity of Step 3 is $O(mr^2T)$.  Therefore, \algo has a total computation complexity $O(\text{nnz}(A)k+\text{nnz}(B)k+mk+mr^2T)$.

\section{Numerical Experiments}\label{experiments}
{\bf Spark implementation.} We implement our \algo in Apache Spark 1.6.2~\cite{spark}. For the purpose of comparison, we also implement a two-pass algorithm \prevalgo~\cite{leverage} in Spark\footnote{To our best knowledge, this the first distributed implementation of \prevalgo.}. The matrices $A$ and $B$ are stored as a resilient distributed dataset (RDD) in disk (by setting its \verb|StorageLevel| as \verb|DISK_ONLY|). We implement the two passes of \prevalgo using the \verb|treeAggregate| method. For \algo, we choose the subsampled randomized Hadamard transform (SRHT)~\cite{srht} as the sketching matrix~\footnote{Compared to Gaussian sketch, SRHT reduces the runtime from $O(ndk)$ to $O(nd\log d)$ and space cost from $O(dk)$ to $O(d)$, while maintains the same quality of the output.}. The biased sampling procedure is performed using binary search (see Appendix~\ref{sec:app_samp} for how to sample $m$ elements in $O(m\log{n})$ time). After obtaining the sampled matrix, we run \verb|ALS| (alternating least squares) to get the desired low-rank matrices. More details can be found in~\cite{github}.

\begin{figure*}[ht]
\centering
\subfigure[]{
	\includegraphics[width=0.37\textwidth]{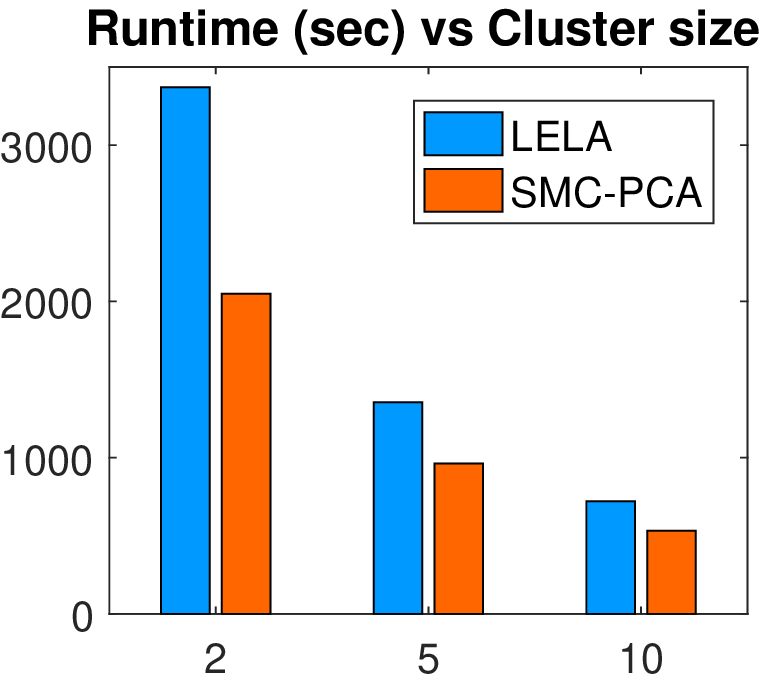}
	\label{runtime}}
\subfigure[]{
	\includegraphics[width=0.59\textwidth]{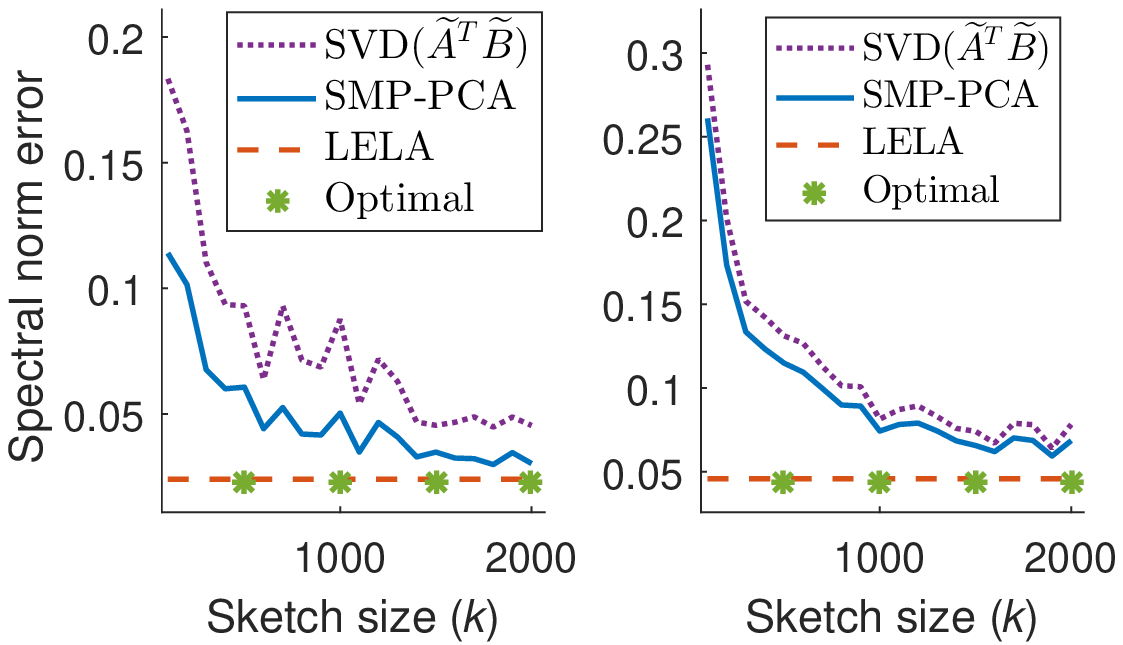}
	\label{siftbag}}
\caption{(a) Spark-1.6.2 running time on a 150GB dataset. All nodes are m.2xlarge EC2 instances. See~\cite{github} for more details. (b) Spectral norm error achieved by three algorithms over two datasets: SIFT10K (left) and NIPS-BW (right). We observe that \algo outperforms $\text{SVD}(\widetilde{A}^T\widetilde{B})$ by a factor of 1.8 for SIFT10K and 1.1 for NIPS-BW. Besides, the error of \algo keeps decreasing as the sketch size $k$ grows.}
\end{figure*}

{\bf Description of datasets.} We test our algorithm on synthetic datasets and three real datasets: SIFT10K~\cite{SIFT}, NIPS-BW~\cite{uci}, and URL-reputation~\cite{URL}. For synthetic data, we generate matrices $A$ and $B$ as $GD$, where $G$ has entries independently drawn from standard Gaussian distribution, and $D$ is a diagonal matrix with $D_{ii} = 1/i$. SIFT10K is a dataset of 10,000 images. Each image is represented by 128 features. We set $A$ as the image-by-feature matrix. The task here is to compute a low rank approximation of $A^TA$, which is a standard PCA task. The NIPS-BW dataset contains bag-of-words features extracted from 1,500 NIPS papers. We divide the papers into two subsets, and let $A$ and $B$ be the corresponding word-by-paper matrices, so $A^TB$ computes the counts of co-occurred words between two sets of papers. The original URL-reputation dataset has 2.4 million URLs. Each URL is represented by 3.2 million features, and is indicated as malicious or benign. This dataset has been used previously for CCA~\cite{url-cca}. Here we extract two subsets of features, and let $A$ and $B$ be the corresponding URL-by-feature matrices. The goal is to compute a low rank approximation of $A^TB$, the cross-covariance matrix between two subsets of features. 

{\bf Sample complexity.} In Figure~\ref{noSamples} we present simulation results on a small synthetic data with $n=d=5,000$ and $r=5$. We observe that a phase transition occurs when the sample complexity $m=\Theta(nr\log{n})$. This agrees with the experimental results shown in previous papers, see, e.g.,~\cite{chen2013completing, leverage}. For the rest experiments presented in this section, unless otherwise specified, we set $r=5$, $T=10$, and sampling complexity $m$ as $4nr\log n$.

\begin{table}[ht]
  \caption{A comparison of spectral norm error over three datasets}
  \label{url}
  \centering
  \begin{tabular}{ c c c c c c}
    \toprule
    Dataset & $d$ & $n$ & Algorithm & Sketch size $k$ & Error \\
    \midrule
    \multirow{3}{4em}{Synthetic} & \multirow{3}{3.5em}{100,000} & \multirow{3}{3.5em}{100,000}  &
    Optimal & - & 0.0271 \\
    & &&\prevalgo & - & 0.0274 \\
    & &&\algo & 2,000  &  0.0280\\
    \midrule
    \multirow{3}{5em}{URL-malicious} & \multirow{3}{4em}{792,145} & \multirow{3}{4em}{10,000} &
    Optimal & - &  0.0163\\
    & &&\prevalgo & - & 0.0182 \\
    & &&\algo & 2,000  & 0.0188 \\
    \midrule
    \multirow{3}{5em}{URL-benign} & \multirow{3}{4em}{1,603,985} & \multirow{3}{4em}{10,000}  &
    Optimal & - &  0.0103\\
    & &&\prevalgo & - & 0.0105\\
    & &&\algo & 2,000  & 0.0117\\
    \bottomrule
  \end{tabular}
\end{table}

{\bf Comparison of \algo and \prevalgo.} We now compare the statistical performance of \algo and \prevalgo~\cite{leverage} on three real datasets and one synthetic dataset. As shown in Figure~\ref{siftbag} and Table~\ref{url}, \prevalgo always achieves a smaller spectral norm error than \algo, which makes sense because \prevalgo takes two passes and hence has more chances exploring the data. Besides, we observe that as the sketch size increases, the error of \algo keeps decreasing and gets closer to that of \prevalgo. 

In Figure~\ref{runtime} we compare the runtime of \algo and \prevalgo using a 150GB synthetic dataset on m3.2xlarge Amazon EC2 instances\footnote{Each machine has 8 cores, 30GB memory, and 2$\times$80GB SSD.}. The matrices $A$ and $B$ have dimension $n=d=100,000$. The sketch dimension is set as $k=2,000$. 
We observe that the speedup achieved by \algo is more prominent for small clusters (e.g., 56 mins versus 34 mins on a cluster of size two). This is possibly due to the increasing spark overheads at larger clusters, see~\cite{overhead} for more related discussion.

{\bf Comparison of \algo and $\text{SVD}(\widetilde{A}^T\widetilde{B})$.} In Figure~\ref{svdOnepass} we repeat the experiment in Section~\ref{algorithms} by generating column vectors of $A$ and $B$ from a cone with angle $\theta$. Here $\text{SVD}(\widetilde{A}^T\widetilde{B})$ refers to computing SVD on the sketched matrices\footnote{This can be done by standard power iteration based method, without explicitly forming the product matrix $\widetilde{A}^T\widetilde{B}$, whose size is too big to fit into memory according to our assumption.}. We plot the ratio of the spectral norm error of $\text{SVD}(\widetilde{A}^T\widetilde{B})$ over that of \algo, as a function of $\theta$. Note that this is different from Figure~\ref{cone}, as now we take the effect of random sampling and SVD into account. However, the trend in both figures are the same: \algo always outperforms $\text{SVD}(\widetilde{A}^T\widetilde{B})$ and can be arbitrarily better as $\theta$ goes to zero.

\begin{figure*}[t]
\centering
\subfigure[]{
	\includegraphics[width=0.31\textwidth]{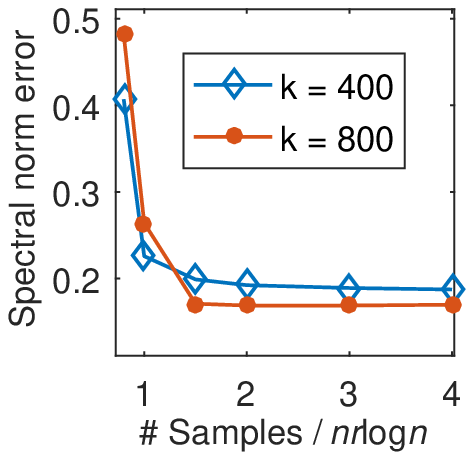}
	\label{noSamples}}
\subfigure[]{
	\includegraphics[width=0.31\textwidth]{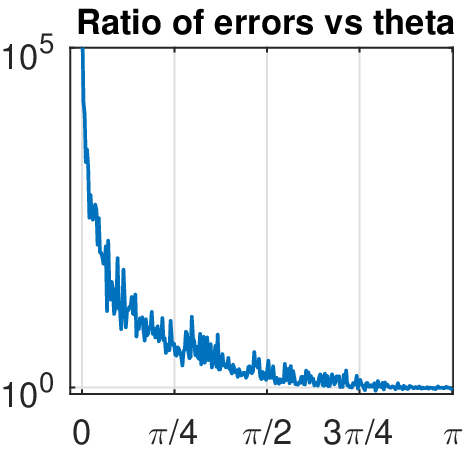}
	\label{svdOnepass}}
\subfigure[]{
	\includegraphics[width=0.31\textwidth]{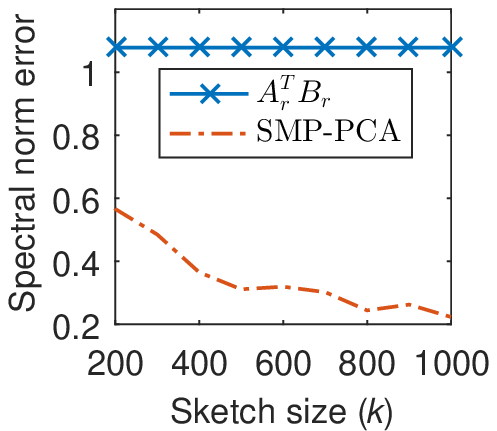}
	\label{svdAsvdB}}
\caption{(a) A phase transition occurs when the sample complexity $m=\Theta(nr\log{n})$. (b) This figure plots the ratio of spectral norm error of $\text{SVD}(\widetilde{A}^T\widetilde{B})$ over that of \algo. The columns of $A$ and $B$ are unit vectors drawn from a cone with angle $\theta$. We see that the ratio of errors scales to infinity as the cone angle shrinks. (c) If the top $r$ left singular vectors of $A$ are orthogonal to those of $B$, the product $A_r^TB_r$ is a very poor low rank approximation of $A^TB$. }
\end{figure*}

In Figure~\ref{siftbag} we compare \algo and $\text{SVD}(\widetilde{A}^T\widetilde{B})$ on two real datasets SIFK10K and NIPS-BW. The y-axis represents spectral norm error, defined as $||A^TB-\widehat{A^TB}_r||/||A^TB||$, where $\widehat{A^TB}_r$ is the rank-$r$ approximation found by a specific algorithm. We observe that \algo outperforms $\text{SVD}(\widetilde{A}^T\widetilde{B})$ by a factor of 1.8 for SIFT10K and 1.1 for NIPS-BW. 

Now we explain why \algo produces a more accurate result than $\text{SVD}(\widetilde{A}^T\widetilde{B})$. The reasons are twofold. First, our rescaled JL embedding $\widetilde{M}$ is a better estimator for $A^TB$ than $\widetilde{A}^T\widetilde{B}$ (Figure~\ref{tildeM}). Second, the noise due to sampling is relatively small compared to the benefit obtained from $\widetilde{M}$, and hence the final result computed using $P_{\Omega}(\widetilde{M})$ still outperforms $\text{SVD}(\widetilde{A}^T\widetilde{B})$. 

{\bf Comparison of \algo and $A_r^TB_r$.} Let $A_r$ and $B_r$ be the optimal rank-$r$ approximation of $A$ and $B$, we show that even if one could use existing methods (e.g., algorithms for streaming PCA) to estimate $A_r$ and $B_r$, their product $A_r^TB_r$ can be a very poor low rank approximation of $A^TB$. This is demonstrated in Figure~\ref{svdAsvdB}, where we intentionally make the top $r$ left singular vectors of $A$ orthogonal to those of $B$. 

\section{Conclusion}
We develop a novel one-pass algorithm~\algo that directly computes a low rank approximation of a matrix product, using ideas of matrix sketching and entrywise sampling. As a subroutine of our algorithm, we propose rescaled JL for estimating entries of $A^T B$, which has smaller error compared to the standard estimator $\tilde{A}^T \tilde{B}$. This we believe can be extended to other applications. Moreover, \algo allows the non-zero entries of $A$ and $B$ to be presented in any arbitrary order, and hence can be used for steaming applications. We design a distributed implementation for \algo. Our experimental results show that \algo can perform arbitrarily better than $\text{SVD}(\widetilde{A}^T\widetilde{B})$, and is significantly faster compared to algorithms that require two or more passes over the data.

{\bf Acknowledgements}
We thank the anonymous reviewers for their valuable comments. This research has been supported by NSF Grants CCF 1344179, 1344364, 1407278, 1422549, 1302435, 1564000, and ARO YIP W911NF-14-1-0258.
\clearpage
\newpage
\bibliographystyle{abbrv}
{\bibliography{references}}

\clearpage
\newpage
\appendix
\onecolumn
\section{Weighted alternating minimization}\label{appWAltMin}
Algorithm~\ref{WAltMin} provides a detailed explanation of WAltMin, which follows a standard procedure for matrix completion. We use $R_{\Omega}(A) = w.*P_\Omega(A)$ to denote the Hadamard product between $w$ and $P_\Omega(A)$: $R_{\Omega}(A)(i,j) = w(i,j)*P_\Omega(A)(i,j)$ for $(i,j)\in\Omega$ and 0 otherwise, where $w\in \mathbb{R}^{n_1\times n_2}=1/\hat{q}_{ij}$ is the weight matrix. Similarly we define the matrix $R_{\Omega}^{1/2}(A)$ as $R_{\Omega}^{1/2}(A)(i,j)= \sqrt{w(i,j)}*P_\Omega(A)(i,j)$ for $(i,j)\in\Omega$ and 0 otherwise.

The algorithm contains two parts: initialization (Step 2-6) and weighted alternating minimization (Step 7-10). In the first part, we compute SVD of the weighted sampled matrix $R_{\Omega}(\widetilde{M})$ and then set row $i$ of $U^{(0)}$ to be zero if its norm is larger than a threshold (Step 6). More details of this trim step can be found in~\cite{leverage}. In the second part, the goal is to solve the following non-convex problem by alternating minimization:
\begin{equation}
\min_{U,V} \sum_{(i,j)\in \Omega} w_{ij}(e_i^T UV^T e_j-\widetilde{M}(i,j))^2,
\end{equation}
where $e_i$, $e_j$ are standard base vectors. After running $T$ iterations, the algorithm outputs a rank-$r$ approximation of $\widetilde{M}$ presented in the convenient factored form. 

\begin{algorithm}[ht]
   \caption{WAltMin~\cite{leverage}}
   \label{WAltMin}
\begin{algorithmic}[1]
   \STATE {\bfseries Input:} $P_\Omega(\widetilde{M})\in \mathbb{R}^{n_1\times n_2}$, $\Omega$, $r$, $\hat{q}$, and $T$ 
   \STATE $w_{ij}=1/\hat{q}_{ij}$ when $\hat{q}_{ij} >0$, $0$ else, $\forall i, j$
   \STATE Divide $\Omega$ in $2T+1$ equal uniformly random subsets, i.e., $\Omega=\{\Omega_0, \dots, \Omega_{2T}\}$
   \STATE $R_{\Omega_0}(\widetilde{M}) = w.*P_{\Omega_0}(\widetilde{M})$
   \STATE $U^{(0)} \Sigma^{(0)} (V^{(0)})^T=\text{SVD}(R_{\Omega_0}(\widetilde{M}), r)$ \\
   \STATE Trim $U^{(0)}$ and let $\widehat{U}^{(0)}$ be the output 
   \FOR {$t=0$ to $T-1$}
	\STATE $\Vht= \argmin_{V }\| R_{\Omega_{2t+1}}^{1/2}(\widetilde{M}- \widehat{U}^{(t)} V^T)\|_F^2$
	\STATE $\widehat{U}^{(t+1)}=\argmin_{U}\| R_{\Omega_{2t+2}}^{1/2}(\widetilde{M}-U(\widehat{V}^{(t+1)})^T)\|_F^2$
	\ENDFOR
    \STATE {\bfseries Output:} $\widehat{U}^{(T)}\in \mathbb{R}^{n_1\times r}$ and $\widehat{V}^{(T)}\in \mathbb{R}^{n_2\times r}$. 
\end{algorithmic}
\end{algorithm}

\section{Technical Lemmas}
We will frequently use the following concentration inequality in the proof.
\begin{lemma} (\emph{Matrix Bernstein's Inequality}~\cite{userFriendly}). \label{Bern}
Consider $p$ independent random matrices $X_1, ...., X_p$ in $\mathbb{R}^{n\times n}$, where each matrix has bounded deviation from its mean: 
\begin{displaymath}
||X_i-\mathbb{E}[X_i]||\le L, \quad \forall i.
\end{displaymath} 
Let the norm of the covariance matrix be 
\begin{displaymath}
\sigma^2 = \max\left\{ \left|\left| \mathbb{E}\left[ \sum_{i=1}^p (X_i -\mathbb{E}[X_i])(X_i -\mathbb{E}[X_i])^T\right] \right|\right|,   \left|\left|  \mathbb{E}\left[ \sum_{i=1}^p (X_i -\mathbb{E}[X_i])^T(X_i -\mathbb{E}[X_i])\right] \right|\right| \right\}
\end{displaymath}
Then the following holds for all $t\ge 0$:
\begin{displaymath}
\mathbb{P}\left[ \left|\left| \sum_{i=1}^p (X_i -\mathbb{E}[X_i])  \right|\right| \right] \le 2n \exp(\frac{-t^2/2}{\sigma^2+Lt/3}).
\end{displaymath}
\end{lemma}

A formal definition of JL transform is given below \cite{sarlos}\cite{sketching}. 
\begin{definition}
A random matrix $\Pi \in \mathbb{R}^{k\times d}$ forms a JL transform with parameters $\epsilon, \delta, f$ or JLT($\epsilon, \delta, f$) for short, if with probability at least $1-\delta$, for any $f$-element subset $V\subset \mathbb{R}^d$, for all $v,v'\in V$ it holds that $|\langle\Pi v, \Pi v' \rangle-\langle v, v' \rangle| \le \epsilon ||v||\cdot||v'||$.\label{JLT}
\end{definition}

The following lemma \cite{sketching} characterizes the tradeoff between the reduced dimension $k$ and the error level $\epsilon$.
\begin{lemma}
Let $0<\epsilon, \delta<1$, and $\Pi\in \mathbb{R}^{k\times d}$ be a random matrix where the entries $\Pi(i,j)$ are i.i.d. $\mathcal{N}(0, 1/k)$ random variables. If $k=\Omega(\log(f/\delta)\epsilon^{-2})$, then $\Pi$ is a JLT($\epsilon, \delta, f$). \label{randomJLT}
\end{lemma}

We now present two lemmas that connect $\widetilde{A}\in \mathbb{R}^{k\times n}$ and $\widetilde{B}\in \mathbb{R}^{d\times n}$ with $A\in\mathbb{R}^{d\times n}$ and $B\in\mathbb{R}^{d\times n}$. 
\begin{lemma} \label{JL-AB}
Let $0<\epsilon, \delta<1$, if $k=\Omega(\frac{\log(2n/\delta)}{\epsilon^2})$, then with probability at least $1-\delta$,
\begin{displaymath}
(1-\epsilon)||A||^2_F \le ||\widetilde{A}||^2_F \le (1+\epsilon) ||A||^2_F, \quad (1-\epsilon)||B||^2_F \le ||\widetilde{B}||^2_F \le (1+\epsilon) ||B||^2_F,
\end{displaymath}
\begin{displaymath}
||\widetilde{A}^T\widetilde{B} - A^TB||_F \le \epsilon ||A||_F ||B||_F.
\end{displaymath} 
\end{lemma}
\begin{proof}
This is again a standard result of JL transformation, e.g., see Definition 2.3 and Theorem 2.1 of \cite{sketching} and Lemma 6 of \cite{sarlos} .
\end{proof}

\begin{lemma} \label{stablerank-AB}
Let $0<\epsilon, \delta<1$, if $k = \Theta (\frac{\tilde{r}+\log(1/\delta)}{\epsilon^2})$, where $\tilde{r} = \max\{\frac{||A||^2_F}{||A||^2}, \frac{||B||^2_F}{||B||^2}\} $ is the maximum stable rank, then with probability at least $1-\delta$,
\begin{displaymath}
||\widetilde{A}^T\widetilde{B} - A^TB|| \le \epsilon ||A|| ||B||.
\end{displaymath}
\end{lemma}
\begin{proof}
This follows from a recent paper \cite{stablerank}.
\end{proof}

Using the above two lemmas, we can prove the following two lemmas that relate $\widetilde{M}$ with $A^TB$, for $\widetilde{M}$ defined in Algorithm~\ref{OnePassLELA}. A more compact definition of $\widetilde{M}$ is $D_A\widetilde{A}^T\widetilde{B}D_B$, where $D_A $ and $D_B$ are diagonal matrices with $(D_A)_{ii}= ||A_i||/||\widetilde{A}_i||$ and $(D_B)_{jj} = ||B_j||/||\widetilde{B}_j||$. 

\begin{lemma} \label{JL}
Let $0<\epsilon<1/14$, $0<\delta<1$, if $k=\Omega(\frac{\log(2n/\delta)}{\epsilon^2})$, then with probability at least $1-\delta$,
\begin{displaymath}
|\widetilde{M}_{ij} - A_i^TB_j|\le \epsilon ||A_i||\cdot ||B_j||,\quad ||\widetilde{M} - A^TB||_F \le \epsilon ||A||_F ||B||_F.
\end{displaymath} 
\end{lemma}

\begin{proof}
Let $0<\epsilon<1/2$, $0<\delta<1$, according to the Definition~\ref{JLT} and Lemma~\ref{randomJLT}, we have that if $k=\Omega(\frac{\log(2n/\delta)}{\epsilon^2})$, then with probability at least $1-\delta$, and for all $i,j$
\begin{equation}
1-\epsilon \le (D_A)_{ii}\le 1+\epsilon, \quad 1-\epsilon \le (D_B)_{jj} \le 1+\epsilon, \quad |\widetilde{A}_i^T \widetilde{B}_j - A_i^TB_j| \le \epsilon ||A_i||||B_j||. \label{JL-sp}
\end{equation}
We can now bound $|\widetilde{M}_{ij} - A_i^TB_j|$ as
\begin{align}
&\quad\;|\widetilde{M}_{ij} - A_i^TB_j| \nonumber\\
& \overset{\xi_1}= |\widetilde{A}_i^T \widetilde{B}_j (D_A)_{ii} (D_B)_{jj} - A_i^TB_j| \nonumber\\
& \overset{\xi_2}\le \max\{|\widetilde{A}_i^T \widetilde{B}_j (1+\epsilon)^2 - A_i^TB_j|,|\widetilde{A}_i^T \widetilde{B}_j (1-\epsilon)^2 - A_i^TB_j| \} \nonumber\\
&\overset{\xi_3} \le \max\{(1+\epsilon)^2 \epsilon ||A_i||||B_j|| + ((1+\epsilon)^2-1) |A_i^TB_j|,(1-\epsilon)^2 \epsilon ||A_i||||B_j|| + (1-(1-\epsilon)^2) |A_i^TB_j| \} \nonumber\\
&\overset{\xi_4} \le 7\epsilon ||A_i||||B_j||,
\end{align}
where $\xi_1$ follows from the definition of $\widetilde{M}_{ij}$, $\xi_2$ follows from the bound in Eq.(\ref{JL-sp}), $\xi_3$ follows from triangle inequality and Eq.(\ref{JL-sp}), and $\xi_4$ follows from $|A_i^TB_j|\le ||A_i||||B_j||$. Now rescaling $\epsilon$ as $\epsilon/7$ gives the desired bound in Lemma~\ref{JL}.

Hence, $||\widetilde{M} - A^TB||_F = \sqrt{\sum_{ij} |\widetilde{M}_{ij} - A_i^TB_j|^2}\le \sqrt{\sum_{ij}\epsilon^2 ||A_i||^2 ||B_j||^2} =  \epsilon ||A||_F ||B||_F$.
\end{proof}

\begin{lemma} \label{stablerank}
Let $0<\epsilon<1/14$, $0<\delta<1$, if $k = \Omega(\frac{\tilde{r}+\log(n/\delta)}{\epsilon^2})$, then with probability at least $1-\delta$,
\begin{displaymath}
||\widetilde{M} - A^TB|| \le \epsilon ||A|| ||B||.
\end{displaymath}
\end{lemma}

\begin{proof}
We can bound the spectral norm of the difference matrix as follows:
\begin{align}
||\widetilde{M} - A^TB|| &\overset{\xi_1} = ||D_A\widetilde{A}^T\widetilde{B}D_B - D_AA^TBD_B +D_AA^TBD_B  - D_AA^TB+D_AA^TB - A^TB|| \nonumber\\
&\le ||D_A|| ||\widetilde{A}^T\widetilde{B} - A^TB|| ||D_B|| + ||D_A||||A^TB||||D_B - I|| + ||D_A- I|| ||A^TB|| \nonumber\\
&\overset{\xi_3} \le (1+\epsilon)^2\epsilon ||A||||B|| + (1+\epsilon)\epsilon ||A||||B|| + \epsilon ||A||||B|| \nonumber \\
&\le 7\epsilon ||A||||B||,
\end{align}
where $\xi_1$ follows from the definition of $\widetilde{M}_{ij}$, and $\xi_2$ follows from Lemma~\ref{stablerank-AB} and bound in Eq.(\ref{JL-sp}). Rescaling $\epsilon$ as $\epsilon/7$ gives the desired bound in Lemma~\ref{stablerank}.
\end{proof}

We will frequently use the term \emph{with high probability}. Here is a formal definition.
\begin{definition}
We say that an event $E$ occurs with high probability (w.h.p.) in $n$ if the probability that its complement $\bar{E}$ happens is polynomially small, i.e., $Pr(\bar{E}) = O(\frac{1}{n^\alpha})$ for some constant $\alpha>0$. 
\end{definition}

The following two lemmas define a "nice" $\Pi$ and when this happens with high probability.
\begin{definition}
The random Gaussian matrix $\Pi$ is "nice" with parameter $\epsilon$ if for all $(i,j)$ such that $q_{ij}\le 1$ (i.e., $q_{ij}=\hat{q}_{ij}$), the sketched values $\widetilde{M}_{ij}$ satisfies the following two inequalities:
\begin{displaymath}
\frac{|\widetilde{M}_{ij}|}{\hat{q}_{ij}}\le (1+\epsilon)\frac{n}{m}(||A||_F^2+||B||_F^2), \quad 
\sum_{\{j: \hat{q}_{ij}=q_{ij}\}} \frac{\widetilde{M}_{ij}^2}{\hat{q}_{ij}} \le (1+\epsilon)\frac{2n}{m}(||A||_F^2+||B||_F^2)^2.
\end{displaymath}
\end{definition}

\begin{lemma}\label{nice}
If $k = \Omega(\frac{\log(n)}{\epsilon^2})$, and $0<\epsilon<1/14$, then the random Gaussian matrix $\Pi \in \mathbb{R}^{k\times d}$ is "nice" w.h.p. in $n$.
\end{lemma}
\begin{proof}
According to Lemma \ref{JL}, if $k = \Omega(\frac{\log(n)}{\epsilon^2})$, then w.h.p. in $n$, for all $(i,j)$ we have $|\widetilde{M}_{ij}-A_i^TB_j|\le \epsilon ||A_i||\cdot ||B_j||$. In other words, the following holds with probability at least $1-\delta$:
\begin{displaymath}
|\widetilde{M}_{ij}|\le |A_i^TB_j| + \epsilon ||A_i||\cdot||B_j|| \le (1+\epsilon)||A_i||\cdot||B_j||, \quad \forall (i,j)
\end{displaymath}
The above inequality is sufficient for $\Pi$ to be "nice":
\begin{displaymath}
\frac{\widetilde{M}_{ij}}{\hat{q}_{ij}} \le (1+\epsilon) \frac{||A_i||\cdot||B_j||}{\hat{q}_{ij}} \le (1+\epsilon) \frac{(||A_i||^2+||B_j||^2)/2}{m \cdot (\frac{||A_i||^2}{2n||A||^2_F}+\frac{||B_j||^2}{2n||B||^2_F})} \le (1+\epsilon)\frac{n}{m}(||A||_F^2+||B||_F^2)
\end{displaymath}
\begin{align}
\sum_{\{j: \hat{q}_{ij}=q_{ij}\}} \frac{\widetilde{M}_{ij}^2}{\hat{q}_{ij}} &\le \sum_{\{j: \hat{q}_{ij}=q_{ij}\}} \frac{(1+\epsilon)^2||A_i||^2||B_j||^2}{\hat{q}_{ij}} \nonumber \\
&\le (1+\epsilon) \sum_{\{j: \hat{q}_{ij}=q_{ij}\}} \frac{||A_i||^4+||B_j||^4}{m \cdot (\frac{||A_i||^2}{2n||A||^2_F}+\frac{||B_j||^2}{2n||B||^2_F})} \nonumber\\
&\le (1+\epsilon)\frac{2n}{m}(||A||_F^2+||B||_F^2)^2.\nonumber
\end{align}
Therefore, we conclude that if $k = \Omega(\frac{\log(n)}{\epsilon^2})$, then $\Pi$ is "nice" w.h.p. in $n$.
\end{proof}

\section{Proofs}\label{mainProof}

\subsection{Proof overview}
We now present the key steps in proving Theorem~\ref{mainTheorem}. The framework is similar to that of \prevalgo~\cite{leverage}. 

Our proof proceeds in three steps. In the first step, we show that the sampled matrix provides a good approximation of the actual matrix $A^TB$. The result is summarized in Lemma~\ref{init}. Here $R_{\Omega}(\widetilde{M})$ denotes the sampled matrix weighted by the inverse of sampling probability (see Line 4 of Algorithm~\ref{WAltMin}). Detailed proof can be found in Appendix~\ref{proofInit}. For consistency, we will use $C_i$ ($i=1,2,...$) to denote global constant that can vary from step to step. 

\begin{lemma}  \label{init} \emph{(Initialization)}
Let $m$ and $k$ satisfy the following conditions for sufficiently large constants $C_1$ and $C_2$:
\begin{displaymath}
m\ge C_1\left(\frac{||{A}||^2_F+||{B}||^2_F}{||{A}^T{B}||_F}\right)^2 \frac{n}{\delta^2}\log(n), 
\end{displaymath}
\begin{displaymath}
k\ge C_2 \frac{\tilde{r}+\log(n)}{\delta^2}\cdot\frac{||A||^2||B||^2}{||A^TB||_F^2},
\end{displaymath}
then the following holds w.h.p. in $n$:
\begin{displaymath}
||R_{\Omega}(\widetilde{M})-A^TB|| \le \delta ||A^TB||_F.
\end{displaymath} 
\end{lemma}

In the second step, we show that at each iteration of WAltMin algorithm, there is a geometric decrease in the distance between the computed subspaces $\widehat{U}$, $\widehat{V}$ and the optimal ones $U^*$, $V^*$. The result is shown in Lemma~\ref{dist}. Appendix~\ref{proofDist} provides the detailed proof. Here for any two orthonormal matrices $X$ and $Y$, we define their distance as the principal angle based distance, i.e., $dist(X,Y)=||X_{\bot}^TY||$, where $X_{\bot}$ denotes the subspace orthogonal to $X$.  

\begin{lemma} \label{dist} \emph{(WAltMin Descent)}
Let $k$, $m$, and $T$ satisfy the conditions stated in Theorem~\ref{mainTheorem}. Also, consider the case when $||A^TB-(A^TB)_r||_F\le \frac{1}{576\rho r^{1.5}}||(A^TB)_r||_F$. Let $\hat{U}^{(t)}$ and $\hat{V}^{(t+1)}$ be the $t$-th and ($t$+1)-th step iterates of the WAltMin procedure. Let $U^{(t)}$ and $V^{(t+1)}$ be the corresponding orthonormal matrices. Let $||(U^{(t)})^i|| \le 8\sqrt{r}\rho ||A_i||/||A||_F$ and $dist(U^{(t)}, U^*)\le 1/2$. Denote $A^TB$ as $M$, then the following holds with probability at least $1- \gamma /T$:
\begin{displaymath}
dist(V^{t+1}, V^*) \le \frac{1}{2} dist(U^t, U^*) +  \eta ||M-M_r||_F/\sigma_r^* + \eta,
\end{displaymath} 
\begin{displaymath}
||(V^{(t+1)})^j|| \le 8\sqrt{r}\rho||B_j||/||B||_F.
\end{displaymath}
\end{lemma}

In the third step, we prove the spectral norm bound in Theorem~\ref{mainTheorem} using results from the above two lemmas. Comparing Lemma~\ref{init} and~\ref{dist} with their counterparts of \prevalgo (see Lemma C.2 and C.3 in~\cite{leverage}), we notice that Lemma~\ref{init} has the same bound as that of \prevalgo, but the bound in Lemma~\ref{dist} contains an extra term $\eta$. This term eventually leads to an additive error term $\eta \sigma_r^*$ in Eq.(\ref{errBound}). Detailed proof is in Appendix~\ref{proofMain}.


\subsection{Proof of Lemma~\ref{init}}\label{proofInit}
We first prove the following lemma, which shows that $R_{\Omega}(\widetilde{M})$ is close to $\widetilde{M}$. For simplicity of presentation, we define $C_{AB}:=\frac{(||A||_F^2+||B||_F^2)^2}{||A^TB||_F^2}$. 
\begin{lemma} \label{C2}
Suppose $\Pi$ is fixed and is "nice". Let $m\ge C_1\cdot C_{AB} \frac{n}{\delta^2}\log(n)$ for sufficiently large global constant $C_1$, then w.h.p. in $n$, the following is true:
\begin{displaymath}
||R_{\Omega}(\widetilde{M})-\widetilde{M}||\le \delta ||A^TB||_F.
\end{displaymath} 
\end{lemma}
\begin{proof}
This lemma can be proved in the same way as the proof of Lemma C.2 in~\cite{leverage}. The key idea is to use the matrix Bernstein inequality. Let $X_{ij} = (\delta_{ij}-\hat{q}_{ij})w_{ij}\widetilde{M}_{ij}e_ie_j^T$, where $\delta_{ij}$ is a $\{0,1\}$ random variable indicating whether the value at $(i,j)$ has been sampled. Since $\Pi$ is fixed, $\{X_{ij}\}_{i,j=1}^n$ are independent zero mean random matrices. Furthermore, $\sum_{i,j}\{X_{ij}\}_{i,j=1}^n=R_{\Omega}(\widetilde{M})-\widetilde{M}$.

\vspace{0.2in}
Since $\Pi$ is "nice" with parameter $0<\epsilon<1/14$, we can bound the 1st and 2nd moment of $X_{ij}$ as follows:
\begin{displaymath}
||X_{ij}|| = \max\{|(1-\hat{q}_{ij})w_{ij}\widetilde{M}_{ij}|, |\hat{q}_{ij}w_{ij}\widetilde{M}_{ij}|\} \le \frac{|\widetilde{M}_{ij}|}{\hat{q}_{ij}} \overset{\xi_1}\le (1+\epsilon)\frac{n}{m}(||A||_F^2+||B||_F^2);
\end{displaymath}
\begin{align}
\sigma^2 &= \max\{\left|\left| \mathbb{E}\left[\sum_{ij}X_{ij}X_{ij}^T\right]\right|\right|, \left|\left| \mathbb{E}\left[\sum_{ij}X_{ij}^TX_{ij}\right]\right|\right|\} \overset{\xi_2}= \max_i \left|\sum_j \hat{q}_{ij}(1-\hat{q}_{ij})w_{ij}^2\widetilde{M}_{ij}^2 \right| \nonumber\\
& = \max_i |(\frac{1}{\hat{q}_{ij}}-1)\widetilde{M}_{ij}^2| \overset{\xi_3}\le \sum_{\{j: \hat{q}_{ij}=q_{ij}\}} \frac{\widetilde{M}_{ij}^2}{\hat{q}_{ij}}\overset{\xi_1} \le (1+\epsilon)\frac{2n}{m}(||A||_F^2+||B||_F^2)^2, \nonumber
\end{align}
where $\xi_1$ follows from Lemma~\ref{nice}, $\xi_2$ follows from a direct calculation, and $\xi_3$ follows from the fact that $\hat{q}_{ij}\le 1$. Now we can use matrix Bernstein inequality (see Lemma~\ref{Bern}) with $t=\delta ||A^TB||_F$ to show that if $m\ge (1+\epsilon)C_1 C_{AB} \frac{n}{\delta^2}\log(n)$, then the desired inequality holds w.h.p. in $n$, where $C_1$ is some global constant independent of $A$ and $B$. Note that since $0<\epsilon<1/14$, $(1+\epsilon)<2$. Rescaling $C_1$ gives the desired result.
\end{proof}

Now we are ready to {\bf prove Lemma~\ref{init}}, which is a counterpart of Lemma C.2 in~\cite{leverage}.
\begin{proof}
We first show that $||R_{\Omega}(\widetilde{M})-\widetilde{M}||\le \delta ||A^TB||_F$ holds w.h.p. in $n$ over the randomness of  $\Pi$. Note that in Lemma~\ref{C2}, we have shown that it is true for a fixed and "nice" $\Pi$, now we want to show that it also holds w.h.p. in $n$ even for a random chosen $\Pi$. 

Let $G$ be the event that we desire, i.e., $G=\{||R_{\Omega}(\widetilde{A}^T\widetilde{B})-\widetilde{A}^T\widetilde{B}||\le \delta ||A^TB||_F\}$. Let $\bar{G}$ be the complimentary event. By conditioning on $\Pi$, we can bound the probability of $\bar{G}$ as
\begin{align}
Pr(\bar{G}) & = Pr(\bar{G}| \Pi \text{ is "nice"}) Pr(\Pi \text{ is "nice"}) + Pr(\bar{G}| \Pi \text{ is not "nice"}) Pr(\Pi \text{ is not "nice"}) \nonumber\\
&\le Pr(\bar{G}| \Pi \text{ is "nice"}) + Pr(\Pi \text{ is not "nice"}). \nonumber
\end{align}
According to Lemma \ref{C2} and Lemma \ref{nice}, if $m\ge C_1\cdot C_{AB} \frac{n}{\delta^2}\log(n)$, and $k\ge C_2\frac{\log(n)}{\epsilon^2}$, then both events $\{G| \Pi \text{ is "nice"}\}$ and $Pr(\Pi \text{ is "nice"})$ happen w.h.p. in $n$. Therefore, the the probability of $\bar{G}$ is polynomially small in $n$, i.e., the desired event $G$ happens w.h.p. in $n$.

Next we show that $||\widetilde{M}-A^TB|| \le \delta ||A^TB||_F$ holds w.h.p. in $n$. According to Lemma~\ref{stablerank}, if $k = \Theta (\frac{\tilde{r}+\log(n)}{\epsilon^2})$, then w.h.p. in $n$, we have $||\widetilde{M}-A^TB|| \le \epsilon ||A|| ||B||$. Now let $\epsilon:=\delta \frac{||A^TB||_F}{||A|| ||B||}$, we have that if $k=\Theta(\frac{\tilde{r}+\log(n)}{\delta^2}\cdot\frac{||A||^2||B||^2}{||A^TB||_F^2})$, then $||\widetilde{M}-A^TB|| \le \delta ||A^TB||_F$ holds w.h.p. in $n$.

By triangle inequality, we have $||R_{\Omega}(\widetilde{M})-A^TB||\le ||R_{\Omega}(\widetilde{M})-\widetilde{M}|| + ||\widetilde{M}-A^TB||$. We have shown that w.h.p. in $n$, both terms are less than $\delta ||A^TB||_F$. By rescaling $\delta$ as $\delta/2$, we have that the desired inequality $||R_{\Omega}(\widetilde{A}^T\widetilde{B})-A^TB|| \le \delta ||A^TB||_F$ holds w.h.p. in $n$, when $m$ and $k$ are chosen according to the statement of Lemma~\ref{init}. 
\end{proof}

Because the bound of Lemma~\ref{init} has the same form as that of Lemma C.2 in \cite{leverage}, the corollary of Lemma C.2 also holds for $R_{\Omega}(\widetilde{M})$, which is stated here without proof: if $||A^TB-(A^TB)_r||_F\le \frac{1}{576\kappa r^{1.5}} ||(A^TB)_r||_F$, then w.h.p. in $n$ we have
\begin{displaymath}
||(\widehat{U}^{(0)})^i|| \le 8\sqrt{r} ||A_i||/||A||_F\quad \text{and} \quad dist(\widehat{U}^{(0)}, U^*) \le 1/2,
\end{displaymath}
where $\widehat{U}^{(0)}$ is the initial iterate produced by the WAltMin algorithm (see Step 6 of Algorithm~\ref{WAltMin}). This corollary will be used in the proof of Lemma~\ref{dist}.

Similar to the original proof in \cite{leverage}, we can now consider two cases separately: (1) $||A^TB-(A^TB)_r||_F\ge \frac{1}{576\rho r^{1.5}} ||(A^TB)_r||_F$; (2) $||A^TB-(A^TB)_r||_F\le \frac{1}{576\rho r^{1.5}} ||(A^TB)_r||_F$. The first case is simple: use Lemma~\ref{init} and Wely's inequality \cite{terrytao} already implies the desired bound in Theorem \ref{mainTheorem}. To see why, note that Lemma~\ref{init} and Wely's inequality imply that
\begin{align}
&\quad\;||(A^TB)_r - (R_{\Omega}(\widetilde{M})_r|| \nonumber\\
&\overset{\xi_1}\le ||A^TB-(A^TB)_r||+ ||A^TB-R_{\Omega}(\widetilde{M})|| + ||R_{\Omega}(\widetilde{M})-(R_{\Omega}(\widetilde{M}))_r|| \nonumber\\
& \overset{\xi_2}\le ||A^TB-(A^TB)_r||+\delta ||A^TB||_F + ||R_{\Omega}(\widetilde{M}) - A^TB|| + ||A^TB- (A^TB)_r|| \nonumber\\
&\overset{\xi_3} \le 2||A^TB- (A^TB)_r||  + 2\delta ||A^TB||_F, \label{firstCase}
\end{align}
where $M_r$ denotes the best rank-$r$ approximation of $M$, $\xi_1$ follows triangle inequality,  $\xi_2$ follows from Lemma~\ref{init} and Wely's inequality, and $\xi_3$ follows from Lemma~\ref{init}. If $||A^TB-(A^TB)_r||_F\ge \frac{1}{576\rho r^{1.5}} ||(A^TB)_r||_F$, then $||A^TB||_F= ||(A^TB)_r||_F+||A^TB-(A^TB)_r||_F\le O(\rho r^{1.5}) ||A^TB-(A^TB)_r||_F$. Setting $\delta = O(\eta/(\rho r^{1.5}))$ in Eq.(\ref{firstCase}) gives the desired error bound in Theorem~\ref{mainTheorem}. Therefore, in the following analysis we only need to consider the second case.

\subsection{Proof of Lemma~\ref{dist}}\label{proofDist}

We first prove the following lemma, which is a counterpart ofLemma C.5 in~\cite{leverage}. For simplicity of presentation, we use $M$ to denote $A^TB$ in the following proof.

\begin{lemma} \label{Thm2}
If $m\ge C_1 nr\log(n)T/(\gamma\delta^2)$ and $k \ge C_2(\tilde{r}+\log(n))/\epsilon^2$ for sufficiently large global constants $C_1$ and $C_2$, then the following holds with probability at least $1-\gamma/T$:
\begin{displaymath}
||(U^{(t)})^TR_{\Omega}(\widetilde{M}-M_r) -(U^{(t)})^T(M-M_r)||\le \delta ||M-M_r||_F + \delta \epsilon ||A||_F||B||_F + \epsilon ||A||||B||.
\end{displaymath} 
\end{lemma}
\begin{proof}
For a fixed $\Pi$, we have that if $m\ge C_1 nr\log(n)T/(\gamma\delta^2)$, then following holds with probability at least $1-\gamma/T$:
\begin{equation}
||(U^{(t)})^TR_{\Omega}(\widetilde{M}-M_r) -(U^{(t)})^T(\widetilde{M}-M_r)||\le \delta ||\widetilde{M}-M_r||_F.  \label{UM}
\end{equation}
The proof of Eq.(\ref{UM}) is exactly the same as the proof of Lemma C.5/B.6/B.2 in~\cite{leverage}, so we omit its details here. The key idea is to define a set of zero-mean random matrices $X_{ij}$ such that $\sum_{ij} X_{ij} = (U^{(t)})^TR_{\Omega}(\widetilde{M}-M_r) -(U^{(t)})^T(\widetilde{M}-M_r)$, and then use second moment-based matrix Chebyshev inequality to obtain the desired bound.

According to Lemma \ref{JL} and Lemma \ref{stablerank}, if $k=\Theta((\tilde{r}+\log(n))/\epsilon^2)$, then w.h.p. in $n$, the following holds:
\begin{equation}
||\widetilde{M} - A^TB||_F \le \epsilon ||A||_F ||B||_F, \quad ||\widetilde{M} - A^TB|| \le \epsilon ||A|| ||B||. \label{JLS}
\end{equation}

Using triangle inequality, we have that if $m$ and $k$ satisfy the conditions of Lemma~\ref{Thm2}, then the following holds with probability at least $1-\gamma/T$:
\begin{align}
&\quad\;||(U^{(t)})^TR_{\Omega}(\widetilde{M}-M_r) -(U^{(t)})^T(M-M_r)|| \nonumber\\
& \le ||(U^{(t)})^TR_{\Omega}(\widetilde{M}-M_r) -(U^{(t)})^T(\widetilde{M}-M_r)|| + ||(U^{(t)})^T(M-\widetilde{M})||\nonumber\\
& \overset{\xi_1}\le \delta ||\widetilde{M}-M_r||_F +  ||M-\widetilde{M}||\nonumber\\
&\le  \delta ||M-M_r||_F + \delta||M-\widetilde{M}||_F + ||M-\widetilde{M}|| \nonumber\\
&\overset{\xi_2}\le \delta ||M-M_r||_F + \delta \epsilon ||A||_F||B||_F + \epsilon ||A||||B||, \nonumber
\end{align}
where $\xi_1$ follows from Eq.(\ref{UM}), and $\xi_2$ follows from Eq.(\ref{JLS}).
\end{proof}



Now we are ready to {\bf prove Lemma~\ref{dist}}. For simplicity, we focus on the rank-1 case here. Rank-$r$ proof follows a similar line of reasoning and can be obtained by combining the current proof with the rank-$r$ analysis in the original proof of \prevalgo~\cite{leverage}. Note that compared to Lemma C.5 in~\cite{leverage}, Lemma~\ref{Thm2} contains two extra terms $\delta \epsilon ||A||_F||B||_F + \epsilon ||A||||B||$. Therefore, we need to be careful for steps that involve Lemma~\ref{Thm2}. 

In the rank-1 case, we use $\hat{u}^t$ and $\hat{v}^{t+1}$ to denote the $t$-th and ($t$+1)-th step iterates (which are column vectors in this case) of the WAltMin algorithm. Let $u^t$ and $v^{t+1}$ be the corresponding normalized vectors. 

\begin{proof}
This proof contains two parts. In the first part, we will prove that the distance $dist(v^{t+1}, v^*)$ decreases geometrically over time. In the second part, we show that the $j$-th entry of $v^{t+1}$ satisfies $|v_j^{t+1}| \le c_1 ||B_j||/||B||_F$, for some constant $c_1$.

{\bf Bounding $dist(v^{t+1}, v^*)$}:

In Lemma~\ref{Thm2}, set $\epsilon = \frac{||A^TB||}{2||A||||B||}\eta$ and $\delta = \frac{\eta}{2\tilde{r}}$, where $0<\eta<1$, then we have $\delta\epsilon ||A||_F||B||_F \le \frac{||A||_F||B||_F}{||A||||B||}\cdot \frac{\eta^2}{2\tilde{r}}||A^TB|| \le  \eta ||A^TB||/2$, and $\epsilon ||A|| ||B|| \le  \eta ||A^TB||/2$.
Therefore, with probability at least $1-\gamma/T$, the following holds:
\begin{equation}
||(u^{t})^TR_{\Omega}(\widetilde{M}-M_1) -(u^{t})^T(M-M_1)||\le \eta ||M-M_1||_F/\tilde{r} + \eta \sigma_1^*. 
\end{equation}
Hence, we have $||(u^{t})^TR_{\Omega}(\widetilde{M}-M_1)||\le dist(u^{t}, u^*)||M-M_1|| +  \eta ||M-M_1||_F/\tilde{r} + \eta \sigma_1^*$. 

Using the explicit formula for WAltMin update (see Eq.(46) and Eq.(47) in~\cite{leverage}), we can bound $\langle \hat{v}^{t+1}, v^* \rangle$ and $\langle \hat{v}^{t+1}, v_{\bot}^* \rangle$ as follows.
\begin{displaymath}
||\hat{u}^t||\langle \hat{v}^{t+1}, v^* \rangle/\sigma_1^* \ge \langle u^t, u^* \rangle - \frac{\delta_1}{1-\delta_1} \sqrt{1-\langle u^t, u^* \rangle^2} - \frac{1}{1-\delta_1}(\eta \frac{||M-M_1||_F}{\tilde{r}\sigma_1^*} + \eta).
\end{displaymath}
\begin{displaymath}
||\hat{u}^t||\langle \hat{v}^{t+1}, v_{\bot}^* \rangle/\sigma_1^*  \le  \frac{\delta_1}{1-\delta_1}\sqrt{1-\langle u^t, u^* \rangle^2} +  \frac{1}{1-\delta_1} (dist(u^{t}, u^*)\frac{||M-M_1||}{\sigma_1^*} +  \eta \frac{||M-M_1||_F}{\tilde{r}\sigma_1^*} + \eta).
\end{displaymath}

As discussed in the end of Appendix~\ref{proofInit}, we only need to consider the case when $||A^TB-(A^TB)_r||_F \le \frac{1}{576\rho r^{1.5}}||(A^TB)_r||_F$, where $\rho = \sigma_1^*/\sigma_r^*$. In the rank-1 case, this condition reduces to $||M-M_1||_F\le \frac{\sigma^*}{576}$. For sufficiently small constants $\delta_1$ and $\eta$ (e.g., $\delta_1\le \frac{1}{20}$, $\eta\le \frac{1}{20}$), and use the fact that $\langle u^t, u^*\rangle \ge \langle u^0, u^* \rangle$ and $dist(u^0,u^*)\le 1/2$, we can further bound $\langle \hat{v}^{t+1}, v^* \rangle$ and $\langle \hat{v}^{t+1}, v_{\bot}^* \rangle$ as 
\begin{equation}
||\hat{u}^t||\langle \hat{v}^{t+1}, v^* \rangle/\sigma_1^* \ge \langle u^0, u^*\rangle -  \frac{1}{10} \sqrt{1-\langle u^0, u^* \rangle^2}  - \frac{1}{10} \ge \frac{\sqrt{3}}{2}-\frac{2}{10} \ge \frac{1}{2}. \label{v_align}
\end{equation}
\begin{align}
||\hat{u}^t||\langle \hat{v}^{t+1}, v_{\bot}^* \rangle/\sigma_1^* &\le \frac{\delta_1}{1-\delta_1} dist(u^t, u^*)+\frac{1}{576(1-\delta_1)}dist(u^t,u^*) + \frac{1}{1-\delta_1}(\eta \frac{||M-M_1||_F}{\tilde{r}\sigma_1^*} + \eta) \nonumber\\
&\overset{\xi_1} \le \frac{1}{4} dist(u^t, u^*) + 2(\eta ||M-M_1||_F/\sigma_1^* + \eta), \label{v_orth}
\end{align}
where $\xi_1$ uses the fact that $\tilde{r}\ge 1$ and the assumption that $\delta_1$ is sufficiently small.

Now we are ready to bound $dist(v^{t+1}, v^*)$ as
\begin{align}
dist(v^{t+1},v^*) &= \sqrt{1-\langle v^{t+1}, v^*\rangle^2} = \frac{\langle \hat{v}^{t+1}, v^*_{\bot}\rangle}{ \sqrt{\langle \hat{v}^{t+1}, v^*_{\bot}\rangle^2+\langle \hat{v}^{t+1}, v^*\rangle^2}} \le \frac{\langle \hat{v}^{t+1}, v^*_{\bot}\rangle}{\langle \hat{v}^{t+1}, v^*\rangle} \nonumber\\
& \overset{\xi_1}\le \frac{1}{2} dist(u^t, u^*) + 4(\eta ||M-M_r||_F/\sigma_1^* + \eta), 
\end{align}
where $\xi_1$ follows from substituting Eqs. (\ref{v_align}) and (\ref{v_orth}). Rescaling $\eta$ as $\eta/4$ gives the desired bound of Lemma~\ref{dist} for the rank-1 case. Rank-r proof can be obtained by following a similar framework.

{\bf Bounding $v_j^{t+1}$}:

In this step, we need to prove that the $j$-th entry of $v^{t+1}$ satisfies $|v_j^{t+1}| \le c_1 \frac{||B_j||}{||B||_F}$ for all $j$, under the assumption that $u^t$ satisfies the norm bound $|u_i^{t}| \le c_1 \frac{||A_i||}{||A||_F}$ for all $i$. 


The proof follows very closely to the second part of proving Lemma C.3 in~\cite{leverage}, except that an extra multiplicative term $(1+\epsilon)$ will show up when bounding $\sum_i \delta_{ij} w_{ij} u_i^t \widetilde{M}_{ij}$ using Bernstein inequality. More specifically, let $X_i = (\delta_{ij} - \hat{q}_{ij})w_{ij}u_i^t \widetilde{M}_{ij}$. Note that if $\hat{q}_{ij}=1$, then $\delta_{ij}=1$, $X_i=0$, so we only need to consider the case when $\hat{q}_{ij}<1$, i.e., $\hat{q}_{ij} = q_{ij}$, where $q_{ij}$ is defined in Eq.(\ref{q_ij}). 

Suppose $\Pi$ is fixed and its dimension satisfies $k = \Omega(\frac{\log(n)}{\epsilon^2})$, then according to Lemma~\ref{JL}, we have that w.h.p. in $n$, 
\begin{equation}
|\widetilde{M}_{ij}|\le |M_{ij}| + \epsilon ||A_i||\cdot||B_j|| \le (1+\epsilon)||A_i||\cdot||B_j||, \quad \forall (i,j). \label{b1}
\end{equation}
Hence, we have 
\begin{equation}
\frac{\widetilde{M}^2_{ij}}{\hat{q}_{ij}} \overset{\xi_1}\le \frac{(1+\epsilon)^2||A_i||^2||B_j||^2}{m \cdot (\frac{||A_i||^2}{2n||A||^2_F}+ \frac{||B_j||^2}{2n||B||^2_F})} \le \frac{2n(1+\epsilon)^2}{m}\cdot ||B_j||^2||A||^2_F, \label{a1}
\end{equation}
\begin{equation}
\frac{(u_i^t)^2}{\hat{q}_{ij}} \overset{\xi_2}\le \frac{c^2_1 ||A_i||^2/||A||^2_F}{m \cdot (\frac{||A_i||^2}{2n||A||^2_F}+ \frac{||B_j||^2}{2n||B||^2_F})} \le \frac{2nc_1^2}{m}, \label{a2}
\end{equation}
where $\xi_1$ follows from substituting Eqs.(\ref{b1}) and (\ref{q_ij}), and $\xi_2$ follows from the assumption that $|u_i^t|\le c_1 ||A_i||/||A||_F$.   

We can now bound the first and second moments of $X_i$ as 
\begin{displaymath}
|X_i| \le |w_{ij}u_i^t\widetilde{M}_{ij}| \le \sqrt{\frac{(u_i^t)^2}{\hat{q}_{ij}}} \sqrt{\frac{\widetilde{M}^2_{ij}}{\hat{q}_{ij}}} \overset{\xi_1}\le \frac{2nc_1(1+\epsilon)}{m} ||B_j||||A||_F.
\end{displaymath}
\begin{align}
\sum_i Var(X_i) &= \sum_i \hat{q}_{ij}(1-\hat{q}_{ij})w^2_{ij} (u_i^t)^2 \widetilde{M}^2_{ij} \le \sum_i \frac{(u_i^t)^2}{\hat{q}_{ij}} (1+\epsilon)^2||A_i||^2||B_j||^2 \nonumber\\
&\overset{\xi_2}\le \frac{2nc^2_1(1+\epsilon)^2}{m}||B_j||^2||A||^2_F, \nonumber
\end{align}
where $\xi_1$ and $\xi_2$ follows from substituting Eqs.(\ref{a1}) and (\ref{a2}).

The rest proof involves applying Bernstein's inequality to derive a high-probability bound on $\sum_i X_i$, which is almost the same as the second part of proving Lemma C.3 in~\cite{leverage}, so we omit the details here. The only difference is that, because of the extra multiplicative term $(1+\epsilon)$ in the bound of the first and second moments, the lower bound on the sample complexity $m$ should also be multiplied by an extra $(1+\epsilon)^2$ term. By restricting $0<\epsilon<1/2$, this extra multiplicative term can be ignored as long as the original lower bound of $m$ contains a large enough constant.
\end{proof}

\subsection{Proof of Theorem~\ref{mainTheorem}}\label{proofMain}
We now prove our main theorem for rank-1 case here. Rank-$r$ proof follows a similar line of reasoning and can be obtained by combining the current proof with the rank-$r$ analysis in the original proof of \prevalgo~\cite{leverage}. Similar to the previous section, we use $\widehat{u}^t$ and $\widehat{v}^{t+1}$ to denote the $t$-th and ($t$+1)-th step iterates (which are column vectors in this case) of the WAltMin algorithm. Let $u^t$ and $v^{t+1}$ be the corresponding normalized vectors. 

The closed form solution for WAltMin update at $t+1$ iteration is
\begin{displaymath}
||\widehat{u}^t|| \widehat{v}_j^{t+1} = \sigma_1^* v_j^* \frac{\sum_i \delta_{ij}w_{ij}u_i^t u_i^*}{\sum_i \delta_{ij}w_{ij}(u_i^t)^2} + \frac{\sum_i \delta_{ij}w_{ij}u_i^t (\widetilde{M}-M_1)_{ij}}{\sum_i \delta_{ij}w_{ij}(u_i^t)^2}.
\end{displaymath}
Writing in matrix form, we get
\begin{equation}
||\widehat{u}^t|| \widehat{v}_j^{t+1} = \sigma_1^*\langle u^*,u^t\rangle v^* - \sigma_1^*B^{-1}(\langle u^*,u^t\rangle B-C)v^* +B^{-1}y, \label{upd}
\end{equation}
where $B$ and $C$ are diagonal matrices with $B_{jj} = \sum_i \delta_{ij}w_{ij}(u_i^t)^2$ and $C_{jj} = \sum_i \delta_{ij}w_{ij}u_i^t u_i^*$, and $y$ is the vector $R_{\Omega}(\widetilde{M}-M_1)^Tu^t$ with entries $y_j = \sum_i \delta_{ij}w_{ij}u_i^t (\widetilde{M}-M_1)_{ij}$.

Each term of Eq.(\ref{upd}) can be bounded as follows.
\begin{equation}
||(\langle u^*,u^t\rangle B -C)v^*|| \le dist(u^t,u^*), \quad ||B^{-1}|| \le 2, \label{sp1}
\end{equation}
\begin{equation}
||y|| = ||R_{\Omega}(\widetilde{M}-M_1)^Tu^t|| \overset{\xi_1}\le dist(u^{t}, u^*)||M-M_1|| +  \eta ||M-M_1||_F/\tilde{r} + \eta \sigma_1^*, \label{sp2}
\end{equation}
where $\xi_1$ follows directly from Lemma~\ref{Thm2}. The proof of Eq.(\ref{sp1}) is exactly the same as the proof of Lemma B.3 and B.4 in~\cite{leverage}.

According to Lemma~\ref{dist}, since the distance is decreasing geometrically, after $O(\log(\frac{1}{\zeta}))$ iterations we get 
\begin{equation}
dist(u^t, u^*) \le \zeta + 2\eta ||M-M_1||_F/\sigma_1^* + 2\eta. \label{sp3}
\end{equation}
 
Now we are ready to prove the spectral norm bound in Theorem~\ref{mainTheorem}:
\begin{align}
&\quad\;||M_1 - \widehat{u}^t (\widehat{v}^{t+1})^T|| \nonumber\\
&\le || M_1 - u^t(u^t)^TM_1|| + ||u^t(u^t)^TM_1-\widehat{u}^t (\widehat{v}^{t+1})^T|| \nonumber\\
&\le || (I - u^t(u^t)^T)M_1|| + ||u^t[(u^t)^TM_1-||\widehat{u}^t|| (\widehat{v}^{t+1})^T]|| \nonumber\\
&\overset{\xi_1}\le \sigma_1^* dist(u^t, u^*) + ||\sigma_1\langle u^t, u^*\rangle v^*-||\widehat{u}^t|| (\widehat{v}^{t+1})^T|| \nonumber \\
& \overset{\xi_2}\le \sigma_1^* dist(u^t, u^*) + ||\sigma_1^*B^{-1}(\langle u^*,u^t\rangle B-C)v^*|| + || B^{-1}y|| \nonumber\\
&\overset{\xi_3} \le \sigma_1^* dist(u^t, u^*) + 2\sigma_1^* dist(u^t,u^*) + 2dist(u^{t}, u^*)||M-M_1|| +  2\eta ||M-M_1||_F/\tilde{r} + 2\eta \sigma_1^* \nonumber\\
&\overset{\xi_4} \le 5 (\zeta \sigma_1^* + 2\eta ||M-M_1||_F + 2\eta \sigma_1^*) + 2\eta ||M-M_1||_F + 2\eta \sigma_1^*  \nonumber\\
& = 5 \zeta \sigma_1^* + 12 \eta ||M-M_1||_F + 12 \eta \sigma_1^* \label{sp4}
\end{align}
where $\xi_1$ follows from the definition of $dist(u^t, u^*)$, the fact that $||u^t||=1$, and $(u^t)^TM_1=\sigma_1\langle u^t, u^*\rangle v^*$, $\xi_2$ follows from substituting Eq.(\ref{upd}), $\xi_3$ follows from Eqs.(\ref{sp1}) and (\ref{sp2}), and $\xi_4$ follows from the Eq.(\ref{sp3}), and fact that $||M-M_1||\le \sigma_1^*$, $\tilde{r}\ge 1$. Rescaling $\zeta$ to $\zeta/ (5\sigma_1^*)$ (this will influence the number of iterations) and also rescaling $\eta$ to $\eta/12$ gives us the desired spectral norm error bound in Eq.(\ref{errBound}). This completes our proof of the rank-1 case. Rank-$r$ proof follows a similar line of reasoning and can be obtained by combining the current proof with the rank-$r$ analysis in the original proof of \prevalgo~\cite{leverage}. 

\subsection{Sampling}\label{sec:app_samp}
We describe a way to sample $m$ elements in $O(m\log(n))$ time using distribution $q_{ij}$ defined in Eq.~\eqref{q_ij}. Naively one can compute all the $n^2$ entries of $\min\{q_{ij},1\}$ and toss a coin for each entry, which takes $O(n^2)$ time. Instead of this binomial sampling we can switch to row wise multinomial sampling. For this, first estimate the expected number of samples per row $m_i = m(\frac{||A_i||^2}{2||A||_F^2}+\frac{1}{2n})$. 
Now sample $m_1$ entries from row $1$ according to the multinomial distribution, 
\begin{displaymath}
\widetilde{q}_{1j} = \frac{m}{m_1}\cdot(\frac{||A_1||^2}{2n||A||^2_F}+\frac{||B_j||^2}{2n||B||^2_F}) = \frac{\frac{||A_1||^2}{2n||A||^2_F}+\frac{||B_j||^2}{2n||B||^2_F}}{\frac{||A_i||^2}{2||A||_F^2}+\frac{1}{2n}}.
\end{displaymath} 
Note that $\sum_j \widetilde{q}_{1j} =1$. To sample from this distribution, we can generate a random number in the interval $[0,1]$, and then locate the corresponding column index by binary searching over the cumulative distribution function (CDF) of $\widetilde{q}_{1j}$. This takes $O(n)$ time for setting up the distribution and $O(m_1 \log(n))$ time to sample. For subsequent row $i$, we only need $O(m_i \log(n))$ time to sample $m_i$ entries. This is because for binary search to work, only  $O(m_i \log(n))$ entries of the CDF vector needs to be computed and checked. Note that the specific form of $\widetilde{q}_{ij}$ ensures that its CDF entries can be updated in an efficient way (since we only need to update the linear shift and scale). Hence, sampling $m$ elements takes a total $O(m \log(n))$ time. Furthermore, the error in this model is bounded up to a factor of 2 of the error achieved by the Binomial model~\cite{candes2009exact}~\cite{kannan2014principal}. For more details please see our Spark implementation.

\section{Related work}\label{sec:related}
{\bf Approximate matrix multiplication:} 

In the seminal work of \cite{drineas2006fast}, Drineas et al. give a randomized algorithm which samples few rows of $A$ and $B$ and computes the approximate product. The distribution depends on the row norms of the matrices and the algorithm achieves an additive error proportional to $||A||_F ||B||_F$. Later Sarlos~\cite{sarlos} propose a sketching based algorithm, which computes sketched matrices and then outputs their product. The analysis for this algorithm is then improved by~\cite{clarkson2009numerical}. All of these results compare the error $||A^TB -\tilde{A}^T \tilde{B}||_F$ in Frobenius norm.

For spectral norm bound of the form $||A^T B-C||_2 \leq \eps ||A||_2 ||B||_2$, the authors in~\cite{sarlos, clarkson2013low} show that the sketch size needs to satisfy $O(r/\eps^2)$, where $r=rank(A)+rank(B)$. This dependence on rank is later improved to stable rank in~\cite{magen2011low}, but at the cost of a weaker dependence on $\eps$. Recently, Cohen et al.~\cite{stablerank} further improve the dependence on $\eps$ and give a bound of $O(\tilde{r}/\eps^2)$, where $\tilde{r}$ is the maximum stable rank. Note that the sketching based algorithm does not output a low rank matrix. As shown in Figure~\ref{tildeM}, rescaling by the actual column norms provide a better estimator than just using the sketched matrices. Furthermore, we show that taking SVD on the sketched matrices gives higher error rate than our algorithm (see Figure~\ref{siftbag}). 

{\bf Low rank approximation:}
\cite{frieze2004fast} introduced the problem of computing low rank approximation of a given matrix using only few passes over the data. They gave an algorithm that samples few rows and columns of the matrix and computes its SVD for low rank approximation. They show that this algorithm achieves additive error guarantees in Frobenius norm. \cite{drineas2006subspace, sarlos, har2006low, deshpande2006adaptive} have later developed algorithms using various sketching techniques like Gaussian projection, random Hadamard transform and volume sampling that achieve relative error in Frobenius norm.\cite{woolfe2008fast, nguyen2009fast, halko2011finding, boutsidis2013improved} improved the analysis of these algorithms and provided error guarantees in spectral norm. More recently \cite{clarkson2013low} presented an algorithm based on subspace embedding that computes the sketches in the input sparsity time.

Another class of methods use entrywise sampling instead of sketching to compute low rank approximation. \cite{achlioptas2001fast} considered an uniform entrywise sampling algorithm followed by SVD to compute low rank approximation. This gives an additive approximation error. More recently \cite{leverage} considered biased entrywise sampling using leverage scores, followed by matrix completion to compute low rank approximation. While this algorithm achieves relative error approximation, it takes two passes over the data. 

There is also lot of interesting work on computing PCA over streaming data under some statistical assumptions, e.g., ~\cite{balsubramani2013fast, mitliagkas2013memory, boutsidis2015online, shamir2015stochastic}. In contrast, our model does not put any assumptions on the input matrix. Besides, our goal here is to get a low rank matrix and not just the subspace.

\end{document}